\documentclass[conference]{IEEEtran}
\usepackage{times}

\usepackage[numbers]{natbib}
\usepackage{multicol}
\usepackage[pagebackref=true,breaklinks=true,colorlinks,bookmarks=false]{hyperref}
\usepackage{amsthm}
\usepackage{adjustbox}
\usepackage{stackengine}
\usepackage{makecell}
\usepackage{tcolorbox}
\usepackage{caption}

\usepackage{comment}
\usepackage{siunitx}
\usepackage{relsize}
\usepackage{ifthen}
\usepackage[colorinlistoftodos]{todonotes}

\usepackage[vlined,ruled,linesnumbered]{algorithm2e}
\usepackage[normalem]{ulem}
\usepackage{graphics} 
\usepackage{rotating}
\usepackage{color}
\usepackage{enumerate}
\usepackage[T1]{fontenc}
\usepackage{psfrag}
\usepackage{epsfig} 
\usepackage{booktabs}
\usepackage{graphicx,url}
\usepackage{multirow}
\usepackage{array}
\usepackage{latexsym}
\usepackage{amsfonts}
\usepackage{amsmath}
\usepackage{amssymb}
\usepackage{mathtools}
\usepackage{xstring}
\usepackage[noend]{algorithmic}
\usepackage{multirow}
\usepackage{xcolor}
\usepackage{prettyref}
\usepackage{flexisym}
\usepackage{bigdelim}
\usepackage{breqn} 
\usepackage{listings}
\let\labelindent\relax
\usepackage{enumitem}
\usepackage{xspace}
\usepackage{bm}
\graphicspath{{./figures/}}
\usepackage{tikz}
\usetikzlibrary{matrix,calc}
\usepackage{tabularx}

\usepackage{pgfplots}
\usepackage{pgfplotstable}
\usepgfplotslibrary{groupplots} 
\pgfplotsset{compat=1.18}       

\newrefformat{proof}{Proof \ref{#1}}

\usepackage{mdwlist}

\let\itemize\undefined
\makecompactlist{itemize}{stditemize}

\newrefformat{prob}{Problem\,\ref{#1}}
\newrefformat{def}{Definition\,\ref{#1}}
\newrefformat{sec}{Section\,\ref{#1}}
\newrefformat{sub}{Section\,\ref{#1}}
\newrefformat{prop}{Proposition\,\ref{#1}}
\newrefformat{app}{Appendix\,\ref{#1}}
\newrefformat{alg}{Algorithm\,\ref{#1}}
\newrefformat{cor}{Corollary\,\ref{#1}}
\newrefformat{thm}{Theorem\,\ref{#1}}
\newrefformat{lem}{Lemma\,\ref{#1}}
\newrefformat{fig}{Fig.\,\ref{#1}}
\newrefformat{tab}{Table\,\ref{#1}}

\newtheorem{theorem}{Theorem}

\newtheorem{lemma}[theorem]{Lemma}

\newtheorem{proposition}[theorem]{Proposition}
\newtheorem{remark}[theorem]{Remark}

\newcommand{\cf}{\emph{cf.}\xspace}

\newcommand{\bdmath}{\begin{dmath}}
\newcommand{\edmath}{\end{dmath}}
\newcommand{\beq}{\begin{equation}}
\newcommand{\eeq}{\end{equation}}
\newcommand{\bdm}{\begin{displaymath}}
\newcommand{\edm}{\end{displaymath}}
\newcommand{\bea}{\begin{eqnarray}}
\newcommand{\eea}{\end{eqnarray}}
\newcommand{\beal}{\beq \begin{array}{ll}}
\newcommand{\eeal}{\end{array} \eeq}
\newcommand{\beas}{\begin{eqnarray*}}
\newcommand{\eeas}{\end{eqnarray*}}
\newcommand{\ba}{\begin{array}}
\newcommand{\ea}{\end{array}}
\newcommand{\bit}{\begin{itemize}}
\newcommand{\eit}{\end{itemize}}
\newcommand{\ben}{\begin{enumerate}}
\newcommand{\een}{\end{enumerate}}

\newcommand{\calE}{{\cal E}}

\newcommand{\calG}{{\cal G}}

\newcommand{\calM}{{\cal M}}

\newcommand{\ie}{\emph{i.e.,}\xspace}

\newcommand{\hide}[1]{}
\newcommand{\wrt}{w.r.t.\xspace}

\newcommand{\hiddenText}{{\color{gray} hidden text.}}
\newcommand{\hideWithText}[1]{\hiddenText}

\newcommand{\kron}{\otimes}

\newcommand{\subject}{\text{ subject to }}

\newcommand{\norm}[1]{\left\| #1 \right\|}

\newcommand{\tran}{^{\mathsf{T}}}

\newcommand{\trace}[1]{\mathrm{tr}\left(#1\right)}

\newcommand{\rank}[1]{\mathrm{rank}\left(#1\right)}

\newcommand{\inv}{^{-1}}

\newcommand{\zero}{{\mathbf 0}}
\newcommand{\eye}{{\mathbf I}}

\newcommand{\Real}[1]{ { {\mathbb R}^{#1} } }

\newcommand{\SEthree}{\ensuremath{\mathrm{SE}(3)}\xspace}

\newcommand{\SOthree}{\ensuremath{\mathrm{SO}(3)}\xspace}
\newcommand{\Othree}{\ensuremath{\mathrm{O}(3)}\xspace}

\newcommand{\scenario}[1]{{\smaller \sf#1}\xspace}

\newcommand{\gtsam}{{\smaller\sf gtsam}\xspace}

\newcommand{\blue}[1]{{\color{blue}#1}}

\newcommand{\linkToPdf}[1]{\href{#1}{\blue{(pdf)}}}
\newcommand{\linkToPpt}[1]{\href{#1}{\blue{(ppt)}}}
\newcommand{\linkToCode}[1]{\href{#1}{\blue{(code)}}}
\newcommand{\linkToWeb}[1]{\href{#1}{\blue{(web)}}}
\newcommand{\linkToVideo}[1]{\href{#1}{\blue{(video)}}}
\newcommand{\linkToMedia}[1]{\href{#1}{\blue{(media)}}}
\newcommand{\award}[1]{\xspace} 

\renewcommand{\norm}[1]{\left\lVert #1 \right\rVert}

\newcommand{\vectorize}[1]{\mathrm{vec}\parentheses{#1}}

\newcommand{\sym}[1]{\mathbb{S}^{#1}}

\newcommand{\bmat}{\left[ \begin{array}}
\newcommand{\emat}{\end{array}\right]}

\newcommand{\parentheses}[1]{\left(#1\right)}

\newcommand{\xmdouble}{\scenario{XM$^2$}}
\newcommand{\nameshort}{\scenario{XM}}
\newcommand{\xmsfm}{\scenario{XM-SfM}}

\newcommand{\sesync}{\scenario{SE-Sync}}

\newcommand{\colmap}{\textsc{Colmap}\xspace}
\newcommand{\glomap}{\textsc{Glomap}\xspace}

\newcommand{\ceres}{\textsc{Ceres}\xspace}
\newcommand{\manopt}{\textsc{Manopt}\xspace}
\renewcommand{\gtsam}{\textsc{Gtsam}\xspace}

\newcommand{\simsync}{\scenario{SIM-Sync}}

\newcommand{\sOthree}{\mathfrak{s}\mathrm{O}(3)}
\newcommand{\qcqp}{\scenario{QCQP}}
\newcommand{\xmbm}{\scenario{BM}}

\newcommand{\sdp}{\scenario{SDP}}

\begin{document}

\title{\vspace{-7mm} Building Rome with Convex Optimization
\vspace{-4mm}}


\author{\authorblockN{
Haoyu Han
and 
Heng Yang
}
\authorblockA{
School of Engineering and Applied Sciences, Harvard University}
\vspace{0.5mm}
\authorblockA{\texttt{\url{https://computationalrobotics.seas.harvard.edu/XM}}}
}



%

\twocolumn[{%
\renewcommand\twocolumn[1][]{#1}%
\maketitle
\vspace{-12mm}
\center

\begin{minipage}{\textwidth}
\centering
\vspace{5mm}
\includegraphics[width=\linewidth]{figures/BAL-demo.pdf}\\
\vspace{-1mm}
(a) BAL Dataset. From left to right: $10155$, $1934$, and $392$ camera frames.\\
\includegraphics[width=\linewidth]{figures/replica-demo.pdf}\\
\vspace{-1mm}
(b) Replica Dataset. All three cases have $2000$ camera frames. \\
\includegraphics[width=\linewidth]{figures/mipnerf-demo.pdf}\\
\vspace{-1mm}
(c) Mip-Nerf Dataset (left: reconstruction, right: novel view synthesis). $185$ camera frames.\\
\includegraphics[width=\linewidth]{figures/imc-demo.pdf}\\
\vspace{-1mm}
(d) IMC Dataset. From left to right: $3765$, $2063$, and $904$ camera frames. \\
\includegraphics[width=\linewidth]{figures/demo-tum.pdf}\\
\vspace{-1mm}
Left (e) TUM Dataset ($798$ and $613$ camera frames); Right (f) C3VD Dataset ($370$ and $613$ camera frames).
\vspace{-1mm}
\captionof{figure}{Faster, scalable, and initialization-free 3D reconstruction powered by conveX bundle adjustMent (\nameshort).
\label{fig:demos}}
\vspace{-1mm}
\end{minipage}

}]


\IEEEpeerreviewmaketitle

 \vspace{-2mm}
\begin{abstract}
   Global bundle adjustment is made easy by depth prediction and convex optimization. We (\emph{i}) propose a scaled bundle adjustment (SBA) formulation that lifts 2D keypoint measurements to 3D with learned depth, (\emph{ii}) design an empirically tight convex semidefinite programming (SDP) relaxation that solves SBA to certifiable global optimality, (\emph{iii}) solve the SDP relaxation at extreme scale with Burer-Monteiro factorization and a CUDA-based trust-region Riemannian optimizer (dubbed \nameshort), (\emph{iv}) build a structure from motion pipeline with \nameshort as the optimization engine and show that \xmsfm compares favorably with existing pipelines in terms of reconstruction quality while being significantly faster, more scalable, and initialization-free. 
\end{abstract}


\section{Introduction}
\label{sec:introduction}

At the heart of modern structure from motion (SfM) and simultaneous localization and mapping (SLAM) sits \emph{bundle adjustment} (BA), the procedure of reconstructing camera poses and 3D landmarks from 2D image keypoints.

\textbf{Classical BA formulation}. Consider a so-called \emph{view graph} illustrated in Fig.~\ref{fig:3d-view-graph} with two types of nodes: 3D points $p_k \in \Real{3}, k=1,\dots,M$ and camera poses $(R_i, t_i) \in \SEthree,i=1,\dots,N$. The set of edges $\calE$ contains visibility information. An edge $(i,k) \in \calE$ indicates the $k$-th 3D point is visible to the $i$-th camera, and a 2D keypoint measurement $u_{ik} \in \Real{2}$ has been obtained regarding the projection of the point onto the camera frame.\footnote{
    We consider BA with calibrated cameras, i.e., $u_{ik}$ has been normalized by camera intrinsics.
} 
The bundle adjustment problem consists of estimating points and poses (i.e., $p_k$'s and $(R_i, t_i)$'s on the nodes) using the 2D keypoint measurements (i.e., $u_{ik}$'s on the edges). This is often formulated as an optimization problem:
\bea\label{eq:ba-colmap}
\min_{\substack{p_k \in \Real{3}, k=1,\dots,M \\ (R_i, t_i) \in \SEthree, i=1,\dots,N } } \sum_{(i,k)\in \calE} \left\| u_{ik} - \pi(R_i p_k + t_i) \right\|^2,
\eea
where $R_i p_k + t_i$ transforms the point $p_k$ to the $i$-th camera frame, $\pi(v):= [v_1 ; v_2]/v_3$ divides the first two coordinates by the depth and projects the 3D point to 2D, and the sum of squared errors evaluates how well the reprojected 2D points agree with the measurements $u_{ik}$ for all $(i,k) \in \calE$. Problem~\eqref{eq:ba-colmap} assumes the availability of a view graph, which is often obtained through feature detection and matching in SfM and SLAM pipelines (to be detailed in \S\ref{sec:xm-sfm}). An important observation is that problem~\eqref{eq:ba-colmap} can only be solved \emph{up to scale}. This is because $\pi(R_i p_k + t_i) \equiv \pi(R_i (\alpha p_k) + (\alpha t_i))$ for any scalar $\alpha > 0$, i.e., scaling the points and translations by a factor of $\alpha$ does not change the objective value of problem~\eqref{eq:ba-colmap}. 

\textbf{Optimization challenges}. Problem~\eqref{eq:ba-colmap} is intuitive to formulate but extremely difficult to optimize. The difficulty comes from two challenges. First, problem~\eqref{eq:ba-colmap} is highly nonconvex. The nonconvexity comes from both the nonconvex feasible set $\SEthree$ and the nonconvex objective function due to the 2D reprojection function $\pi(\cdot)$. Second, problem~\eqref{eq:ba-colmap} can have an extremely large scale. Both the number of camera poses $N$ and the number of points $M$ can range from hundreds to tens of thousands (\cf examples in Fig.~\ref{fig:demos}). Due to these challenges, BA solvers such as \ceres~\cite{agarwal2012ceres}, \gtsam~\cite{dellaert2022gtsam}, and accelerated variants \cite{fan2023decentralization,ren2022megba} require good initializations and can often get stuck in poor local minima (\cf \S\ref{sec:exp-bal}). To address this, the popular SfM pipeline \colmap~\cite{schoenberger2016sfm} employs an \emph{incremental} strategy which starts by reconstructing only two views (i.e.,~\eqref{eq:ba-colmap} with $N=2$) and then sequentially registers additional camera images and associated 3D structure. Incremental SfM ensures stable initialization but makes the pipeline slow and hard to scale to large datasets (\cf \S\ref{sec:exp} where \colmap requires several hours runtime). The recent \emph{global} SfM pipeline \glomap~\cite{pan2025global} replaces the incremental process with rotation averaging and global positioning, but such initialization can still be time-consuming (see results in \S\ref{sec:exp}).

Therefore, the motivating question of this paper is:

\emph{Can we design an algorithm that solves the global bundle adjustment problem~\eqref{eq:ba-colmap} without initialization and at scale?}


\begin{figure}
    \centering
    \includegraphics[width=0.9\linewidth]{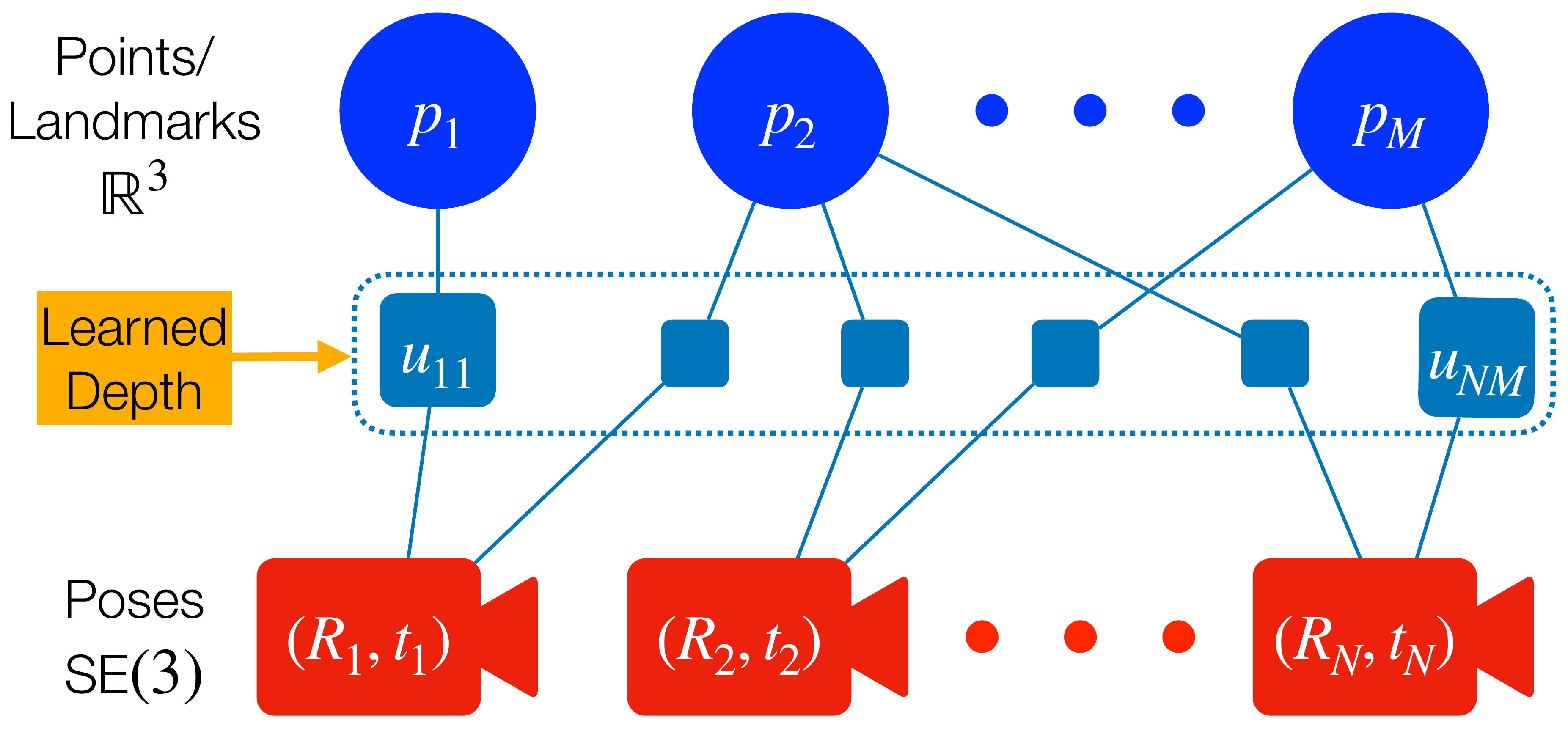}
    \caption{A view graph for the bundle adjustment formulation~\eqref{eq:ba-colmap}. We propose a \emph{scaled bundle adjustment} formulation~\eqref{eq:sba} by lifting the 2D keypoints to 3D with learned depth.}
    \label{fig:3d-view-graph}
    \vspace{-4mm}
\end{figure}

\textbf{BA with learned depth}. We start by designing an approximate formulation to problem~\eqref{eq:ba-colmap} that is simpler to optimize. The key insight is to lift the 2D keypoints $u_{ik}$ as approximate (and noisy) 3D keypoints leveraging off-the-shelf large-scale pretrained monocular depth prediction models, such as~\cite{piccinelli2024unidepth,yang2024depth,bochkovskii2024depth}. Formally, let $d_{ik} > 0$ be the predicted depth of the 2D keypoint $u_{ik}$, we generate a 3D keypoint measurement
\bea\label{eq:3D-lifting}
\tilde{u}_{ik} = d_{ik} \begin{bmatrix}
    u_{ik} \\ 1 
\end{bmatrix}, \quad \forall (i,k) \in \calE.
\eea
In our work, we select a \emph{metric depth} model, which is designed to produce depth estimates that match the true depth when the model is perfectly trained. Further, due to imperfect depth prediction caused by potentially out-of-distribution test images, we incorporate a scaling parameter that allows subsequent optimization to correct erroneous predictions, leading to the following \emph{scaled bundle adjustment} (SBA) formulation
\bea\label{eq:sba}
\min_{\substack{p_k \in \Real{3}, k=1,\dots,M \\ s_i > 0, i=1,\dots,N
\\ (R_i, t_i) \in \SEthree, i=1,\dots,N } } \sum_{(i,k)\in \calE} w_{ik}\left\| R_i (s_i \tilde{u}_{ik}) + t_i - p_k \right\|^2,
\eea
where $s_i$ scales the predicted 3D keypoint $\tilde{u}_{ik}$ in~\eqref{eq:3D-lifting}, $(R_i, t_i)$ transforms the scaled 3D keypoint to the global frame,\footnote{Although we used the same notation $(R_i,t_i)$ in both problems~\eqref{eq:ba-colmap} and~\eqref{eq:sba}, the two transformations are inverse to each other. $(R_i,t_i)$ in~\eqref{eq:ba-colmap} transforms landmarks from global frame to camera frame, while $(R_i,t_i)$ in~\eqref{eq:sba} transforms landmarks from camera frame to global frame.} and the squared errors in the objective computes 3D distances to the landmarks $p_k$ without the reprojection $\pi(\cdot)$. In \eqref{eq:sba}, we allow weighting the squared errors by confidence values ($w_{ik}>0$) from the depth prediction model. 

\begin{remark}[Connection to Other Perception Problems]
    It is worth noting that (i) without the per-frame scaling, problem~\eqref{eq:sba} recovers the multiple point cloud registration problem~\cite{chaudhury2015global,iglesias2020global}; (ii) when $N=2$ (two frames), problem~\eqref{eq:sba} reduces to (scaled) point cloud registration~\cite{horn1987closed,yang2020teaser}.
\end{remark}

Through depth prediction, we removed the reprojection function $\pi(\cdot)$ from~\eqref{eq:ba-colmap} and resolved one challenge. However, nonconvexity and large scale remain in problem~\eqref{eq:sba}.

\textbf{Contributions}. First, we show the remaining nonconvexity in problem~\eqref{eq:sba} is ``benign''. Our strategy is to first rewrite~\eqref{eq:sba} as a quadratically constrained quadratic program (QCQP), and then design a convex semidefinite program (SDP) relaxation, in the spirit of a growing family of SDP-enabled certifiable algorithms~\cite{yang2022certifiably,rosen2019se,yu2024sim,kang2024fast,barfoot2023certifiably,papalia2024certifiably,tian2021distributed}. We show the SDP relaxation is empirically \emph{tight}, i.e.,~globally optimal solutions of the nonconvex~\eqref{eq:sba} can be computed from the convex SDP relaxation with optimality certificates. Second, we show the convex SDP relaxations can be solved at \emph{extreme scale and speed}---faster and more scalable than even the best local solvers such as \ceres. The enabling technique is to exploit the low-rankness of (tight) SDP optimal solutions via Burer-Monteiro (BM) factorization~\cite{burer2003nonlinear} and solve the resulting Riemannian optimization using a trust-region algorithm~\cite{boumal2023introduction}. \emph{For the first time, we implemented the trust-region Riemannian optimizer directly in C++/CUDA} and show the GPU implementation is up to $100$ times faster than the state-of-the-art CPU-based \manopt package~\cite{boumal2014manopt} and can solve extreme-scale problems beyond the reach of \manopt (e.g., $N>10,000$ camera frames, see Fig.~\ref{fig:demos}). We name our GPU solver \nameshort (conveX bundle adjustMent). Third, we build a full SfM pipeline with \nameshort as the optimization engine and various techniques from prior work, such as feature matching from \colmap, view graph creation from \glomap, \ceres refinement (i.e., warmstart \ceres with \nameshort's solutions for solving~\eqref{eq:ba-colmap}), and outlier-robust filtering and estimation schemes~\cite{antonante2021outlier,fischler1981random,olsson2010outlier,sim2006removing}. We test \xmsfm across six popular datasets and demonstrate that \xmsfm compares favorably with existing SfM pipelines in terms of reconstruction quality while being significantly faster, more scalable, and initialization-free, thanks to convex optimization and GPU-based implementation.

In summary, our contributions are:
\begin{itemize}
    \item designing an empirically tight convex SDP relaxation for the scaled bundle adjustment problem~\eqref{eq:sba};
    \item solving the convex SDP at extreme scales using BM factorization paired with a trust-region Riemannian optimizer directly implemented in C++/CUDA, \ie~\nameshort;
    \item creating a full SfM pipeline called \xmsfm that ``builds Rome with convex optimization''.
\end{itemize}

\textbf{Paper organization}. We derive the QCQP formulation for~\eqref{eq:sba} and design the SDP relaxation in \S\ref{sec:method}. We present BM factorization and the CUDA-based trust-region Riemannian optimizer in \S\ref{sec:implementation}. We describe the \xmsfm pipeline in \S\ref{sec:xm-sfm} and present experimental results in \S\ref{sec:exp}. We conclude in \S\ref{sec:conclusion}.
\section{QCQP and Convex SDP Relaxation}
\label{sec:method}

In this section, let us focus on solving the SBA problem~\eqref{eq:sba} to certifiable global optimality. We proceed in two steps. In \S\ref{sec:qcqp}, we simplify the original formulation as a quadratically constrained quadratic program (QCQP) through a sequence of mathematical manipulations. In \S\ref{sec:sdp-relax}, we apply Shor's semidefinite relaxation to ``convexify'' the nonconvex QCQP. 

Before we get started, we remove the ambiguity of problem~\eqref{eq:sba} through anchoring.


{\bf Anchoring}. Observe that one can choose $s_i \rightarrow 0, \forall i = 1, \dots, N$, $t_1 = \dots = t_N = p_1 = \dots = p_M = {\bf 0}$, and the objective value of~\eqref{eq:sba} can be set arbitrarily close to zero. Additionally, multiplying an arbitary rotation matrix on $R_i, t_i, p_i$ does not change the objective of~\eqref{eq:sba}, leading to infinitely many solutions. To resolve these issues, we anchor the first frame and set $R_1 = \eye_3, t_1 = \zero, s_1 = 1$.

\subsection{QCQP Formulation}
\label{sec:qcqp}

We first show that the unconstrained variables in~\eqref{eq:sba}, namely the translations $t_i$ and the 3D landmarks $p_k$, can be ``marginalized out'', leading to an optimization problem only concerning the scaling factors and 3D rotations.

\begin{proposition}[Scaled-Rotation-Only Formulation]
    \label{prop:formulation_scale_rotation_only}
    Problem \prettyref{eq:sba} is equivalent to the following optimization
    \begin{subequations}\label{eq:problem_scale_rotation_only}
        \begin{align}
            \rho^\star = \min & \ \ \trace{Q U\tran U}  \\
            \subject & \ \ U = \bmat{cccc} \eye_3 & s_2R_2 & \cdots & s_N R_N \emat\\
             &\ \ s_i >0, R_i \in \SOthree, \quad i=2,\cdots,N
            \end{align}
    \end{subequations}
    where $Q \in \sym{3N}$ is a constant and symmetric ``data matrix'' whose expression is given in \prettyref{app:proof_scale_rotation_only}.


    
    Let $U^\star$ represent the optimal solution of \prettyref{eq:problem_scale_rotation_only} and let $t = [t_1;\dots;t_N] \in \Real{3N}$ be the concatenation of translations, $p = [p_1;\dots;p_M] \in \Real{3M}$ be the concatenation of landmark positions. The optimal translations $t^\star$ and landmark positions $p^\star$ of problem~\eqref{eq:sba} can be recovered from $U^\star$ as follows:
    \bea \label{eq:tpstar}
    \bmat{c} t^\star \\ p^\star  \emat = (A \kron \eye_3) \vectorize{U^\star}.
    \eea 
    where the expression of $A$ can be found in \prettyref{app:proof_scale_rotation_only}. 
\end{proposition}
\begin{proof}
    See \prettyref{app:proof_scale_rotation_only}.
\end{proof}

Proposition~\ref{prop:formulation_scale_rotation_only} reformulates the SBA problem~\eqref{eq:sba} as a lower-dimensional problem~\eqref{eq:problem_scale_rotation_only}. However, problem~\eqref{eq:problem_scale_rotation_only} is not a QCQP because the objective function is a degree-four polynomial in $s_i$ and $R_i$. In the next step, we show that it is possible to combine the ``scaled rotation'' $s_i R_i$ together as a new variable, effectively reducing the degree of the polynomial.



\begin{proposition}[QCQP Formulation]\label{prop:qcqp}
    Define the set of scaled orthogonal group as
    \bea \label{eq:sOthreedef}
    \sOthree = \{ \bar{R} \in \Real{3\times 3} \mid \exists s > 0, R \in \Othree~\text{s.t.}~ \bar{R} = sR \}.
    \eea 
    The set $\sOthree$ can be described by quadratic constraints
    \begin{equation} \label{eq:sOthreequadratic}
    \begin{split}
        \bar{R} = \bmat{ccc} c_1 & c_2 & c_3 \emat \in \sOthree \Longleftrightarrow \\ \begin{cases}
            c_1\tran c_1 = c_2\tran c_2 = c_3\tran c_3 \\
            c_1\tran c_2 = c_2\tran c_3 = c_3\tran c_1 = 0 .
        \end{cases}
    \end{split}
    \end{equation} 
    Consider the following QCQP
    \begin{subequations}\label{eq:problem-qcqp}
        \begin{align} 
            \rho_{\qcqp}^\star = \min  &\ \ \trace{Q U\tran U}  \\
            \subject & \ \ U = \bmat{cccc} \eye_3 & \bar{R}_2 & \cdots & \bar{R}_N \emat\\
            &\ \ \bar{R}_i \in \sOthree, ~i = 2,\dots,N.
        \end{align}
    \end{subequations}
    and let $U^\star = [\eye_3, \bar{R}_2^\star, \dots,\bar{R}_N^\star]$ be a global optimizer. If 
    \bea \label{eq:determinant}
    \det{\bar{R}_i^\star } > 0, i=2,\dots,N,
    \eea 
    then $U^\star$ is a global minimizer to problem \prettyref{eq:problem_scale_rotation_only}.
\end{proposition}
\begin{proof}
    See \prettyref{app:proof_qcqp}.
\end{proof}

It is clear that 
\bea\label{eq:qcqp-inequality}
\rho^\star_\qcqp \leq \rho^\star
\eea
because from~\eqref{eq:problem_scale_rotation_only} to~\eqref{eq:problem-qcqp} we have relaxed the $\SOthree$ constraint to $\Othree$.
After solving~\eqref{eq:problem-qcqp},  
if there exists some index $i$ such that $\det{\bar{R}_i^\star} < 0$, we can extract a feasible solution to~\eqref{eq:problem_scale_rotation_only} by projecting onto $\SOthree$ after extracting the scaling from $\bar{R}_i^\star$.

\begin{remark}[Bad Local Minimum]
    Problem~\eqref{eq:problem-qcqp} is a nonconvex QCQP. In fact, it is a smooth Riemannian optimization by identifying $\sOthree$ as a product manifold of the positive manifold and the orthogonal group. Therefore, one can directly use, e.g., \manopt to solve~\eqref{eq:problem-qcqp} (and also~\eqref{eq:problem_scale_rotation_only}). However, as we will show in \S\ref{sec:exp} on the Mip-Nerf 360 dataset~\cite{barron2022mipnerf360}, directly solving~\eqref{eq:problem-qcqp} can get stuck in bad local minima. 
\end{remark}

This motivates and necessitates convex relaxation.




\subsection{Convex SDP Relaxation}
\label{sec:sdp-relax}

For any QCQP, there exists a convex relaxation known as Shor's semidefinite relaxation~\cite[Chapter 3]{yang24book-sdp}. The basic idea is fairly simple: by creating a matrix variable $X := U\tran U$ that is quadratic in the original variable $U$, problem~\eqref{eq:problem-qcqp} becomes \emph{linear} in $X$. The convex relaxation proceeds by using convex positive semidefinite constraints and linear constraints to properly enforce the matrix variable $X$. 

\begin{proposition}[SDP Relaxation]\label{prop:sdprelaxation}
    The following semidefinite program (SDP)
    \begin{subequations}\label{eq:problem-sdp}
        \begin{align}
            \rho^\star_{\sdp} = \min_{X \in \sym{3N}} &\ \  \trace{QX}  \\
            \subject & \ \ X = \bmat{ccc} \eye_3 & \cdots & * \\
            \vdots & \ddots & \vdots \\
            * & \cdots & \alpha_N \eye_3
            \emat \succeq 0 \label{eq:sdp-X-compact}
            \end{align}
    \end{subequations}
    is a convex relaxation to \prettyref{eq:problem-qcqp}, i.e., 
    \begin{equation}\label{eq:sdp-inequality}
        \rho^\star_\sdp \leq \rho^\star_\qcqp.
    \end{equation}
    Let $X^\star$ be a global minimizer of \prettyref{eq:problem-sdp}. 
    
    If $\rank{X^\star} = 3$, then $X^\star$ can be factorized as $X^\star = (\bar{U}^\star) \tran \bar{U}^\star$, and it holds\footnote{$\bar{R}_1^\star$ is in $O(3)$ because we restrict the scale of the first frame to be 1. }
    \bea
    \bar{U}^\star = \bmat{cccc} \bar{R}_1^\star & \bar{R}_2^\star & \cdots & \bar{R}_N^\star \emat \in O(3)\times\sOthree^{N-1}. 
    \eea
    Define 
    \bea
    U^\star = (\bar{R}_1^\star)\tran \bar{U}^\star,
    \eea
    then $U^\star$ is a global optimizer to \prettyref{eq:problem-qcqp}. In this case, we say the relaxation is \emph{tight} or \emph{exact}.
\end{proposition}
\begin{proof}
    See \prettyref{app:proof_sdp_relaxation}.
\end{proof}

If $\rank{X^\star} > 3$, then we can extract feasible solutions to~\eqref{eq:problem-qcqp} by taking the top three eigenvectors of $X^\star$ and project the corresponding entries to $\sOthree$ and further to scaled rotations, a step that is typically called \emph{rounding}. Denote $\hat{U}$ as the rounded solution that is feasible for~\eqref{eq:problem_scale_rotation_only}, we can evaluate the objective of~\eqref{eq:problem_scale_rotation_only} at $\hat{U}$ and denote it $\hat{\rho}$. Combining the chain of inequalities from~\eqref{eq:qcqp-inequality} and~\eqref{eq:sdp-inequality}, we get
\bea\label{eq:chain-inequality}
\rho^\star_\sdp \leq \rho^\star_\qcqp \leq \rho^\star \leq \hat{\rho},
\eea 
where the last inequality follows from $\hat{U}$ is a feasible solution for the minimization problem~\eqref{eq:problem_scale_rotation_only}. Assuming we can solve the convex SDP, we compute $\rho^\star_\sdp$ and $\hat{\rho}$ at both ends of the inequality~\eqref{eq:chain-inequality}, allowing us to evaluate a relative suboptimality:
\bea \label{eq:suboptimality}
\eta = \frac{\hat{\rho} - \rho^\star_\sdp}{1 + |\hat{\rho}| + |\rho^\star_\sdp|}.
\eea
$\eta \rightarrow 0$ certifies global optimality of the rounded solution $\hat{U}$, and tightness of the SDP relaxation.

\begin{tcolorbox}
    \begin{center}
        \vspace{-1mm}
        \textbf{Summary}
        \vspace{-1mm}
    \end{center} 
    From~\eqref{eq:sba} to~\eqref{eq:problem_scale_rotation_only}, we first eliminated translations and landmark positions. From~\eqref{eq:problem_scale_rotation_only} to~\eqref{eq:problem-qcqp}, we formulated a QCQP by creating the new constraint set $\sOthree$. From~\eqref{eq:problem-qcqp} to~\eqref{eq:problem-sdp}, we applied Shor's semidefinite relaxation. This sequence of manipulations and relaxations allows us to focus on solving the convex SDP problem~\eqref{eq:problem-sdp} while maintaining the ability to certify (sub)optimality of the original nonconvex problem~\eqref{eq:sba}, through the inequalities established in~\eqref{eq:chain-inequality}. 
\end{tcolorbox}

\begin{remark}[Connection to Prior Work]\label{eq:connection}
    For readers familiar with SDP relaxations, this section should not be surprising at all---that is the reason why we kept this section very brief and only focus on milestone results listed in the propositions. The closest two works to ours are \sesync~\cite{rosen2019se} and \simsync~\cite{yu2024sim}, and the SDP relaxation technique traces back to at least the work by Carlone et al.~\cite{carlone2015lagrangian}. The novelty of our formulation are twofold: (a) we estimate scaling factors with the motivation to correct learned depth while \sesync estimates rotations and translations only; (b) we jointly estimate 3D landmarks and (scaled) camera poses while \simsync does not estimate 3D landmarks. These differences allow us to solve the long-standing bundle adjustment problem with convex optimization.
\end{remark}

We shall focus on how to solve the convex SDP~\eqref{eq:problem-sdp}.

\section{Burer-Monteiro Factorization and \\ CUDA-based Riemannian Optimizer}
\label{sec:implementation}

Let us first write the SDP~\eqref{eq:problem-sdp} in standard primal form:
\begin{subequations}\label{eq:problem-sdp-standard}
    \begin{align} 
        \min_{X \in \sym{3N}} & \ \ \trace{QX} \\
        \subject & \ \ \langle A_i,X \rangle = b_i, i = 1,\dots,m \label{eq:sdp:standard:linear}\\
        & \ \ X \succeq 0  \label{eq:sdp:standard:psd}
        \end{align}
\end{subequations}
where we have rewritten the constraint~\eqref{eq:sdp-X-compact} as a positive semidefinite (PSD) constraint~\eqref{eq:sdp:standard:psd} and $m=5N+1$ linear equality constraints~\eqref{eq:sdp:standard:linear}. To see why this reformulation is true, note that the diagonal $3\times 3$ blocks of $X$ are either $\eye_3$ or scaled $\eye_3$---the former can be enforced using 6 linear equalities and the latter can be enforced using 5, summing up to $5N+1$ linear equalities. Associated with the primal standard SDP~\eqref{eq:problem-sdp-standard} is the following dual standard SDP:
\begin{subequations}\label{eq:problem-sdp-dual}
    \begin{align} 
    \max_{y \in \Real{m}} & \ \ \sum_{i=1}^{m} b_i y_i \\
    \subject & \ \ Z(y) := Q - \sum_{i=1}^{m} y_i A_i \succeq 0. \label{eq:Zofy}
    \end{align} 
\end{subequations}

Since Slater's condition holds ($X = \eye_{3N} \succ 0$ is feasible for the primal~\eqref{eq:problem-sdp-standard}), we know strong duality holds between primal~\eqref{eq:problem-sdp-standard} and dual~\eqref{eq:problem-sdp-dual}, i.e., their optimal values are equal to each other~\cite{yang24book-sdp,wolkowicz2012handbook}.  

\textbf{Scalability}. Once the SDP~\eqref{eq:problem-sdp} is formulated in standard forms, it can be solved by off-the-shelf SDP solvers such as MOSEK~\cite{aps2019mosek}, provided that the SDP's scale is not so large. Our SDP relaxation leads to a matrix variable with size $3 N \times 3N$. This means that when $N$ is in the order of hundreds, MOSEK can solve the SDP without any problem. However, when $N$ is in the order of thousands or tens of thousands, which is not uncommon in bundle adjustment problems, MOSEK will become very slow or even runs out of memory (e.g., MOSEK becomes unresponsive when $N>2000$). Therefore, we decided to customize a solver for our SDP relaxation.




    


\subsection{Burer-Monteiro Factorization}

The key structure we will leverage is, as stated in Proposition~\ref{prop:sdprelaxation}, the optimal solution $X^\star$ has its rank equal to three when the SDP relaxation is tight. In other words, the effective dimension of $X^\star$ can be $3 \times 3N$ instead of $3N \times 3N$.

To exploit the low-rank structure, we will leverage the famous Burer-Monteiro (BM) factorization~\cite{burer2003nonlinear}.

\begin{proposition}[BM Factorization]\label{prop:BM}
    For a fixed rank $r \geq 3$, the Burer-Monteiro factorization of \prettyref{eq:problem-sdp} and~\eqref{eq:problem-sdp-standard} reads:
    \begin{subequations}\label{eq:problem-bm}
        \begin{align} 
            \rho^\star_{\xmbm,r} = \min_{ \bar{R}_i \in \Real{r \times 3},i=1,\dots,N } & \ \ \trace{ Q U\tran U } \\
            \subject & \ \ U = \bmat{ccc}\bar{R}_1& \dots &\bar{R}_N \emat  \\
            & \ \ \bar{R}_i \tran \bar{R}_i = \begin{cases}
                \eye_3 & i=1 \\
                \alpha_i\eye_3 & i\geq 2
            \end{cases}.
            \end{align}
    \end{subequations}
\end{proposition}

A few comments are in order. First, by factorizing $X = U\tran U$, $X \succeq 0$ holds by construction. Second, note that when $r = 3$, problem~\eqref{eq:problem-bm} is exactly the same as the original QCQP~\eqref{eq:problem-qcqp} (up to the difference in $\bar{R}_1$), and thus is NOT convex. Third, as long as $r \geq r^\star$ where $r^\star$ is the minimum rank of all optimal solutions of the SDP~\eqref{eq:problem-sdp}, the nonconvex BM factorization has the same global minimum as the convex SDP~\eqref{eq:problem-sdp}~\cite{yang24book-sdp,burer2003nonlinear}. Note that the factor $U$ is a matrix of size $r \times 3N$, i.e., a flat matrix as shown in Fig.~\ref{fig:pipeline} bottom right.

\textbf{Counterintuitive}? The reader might find this confusing. In \S\ref{sec:method}, we applied a series of techniques to relax a nonconvex optimization problem into a convex SDP, enabling us to solve it to global optimality. Surprisingly, the BM factorization appears to ``undo'' this effort, pulling us back into nonconvex optimization. While the low-rank factorization offers a clear scalability advantage, it remains uncertain whether this benefit outweighs the drawbacks of reintroducing nonconvexity.

\textbf{Rank ``staircase''}. The secret ingredient of BM factorization to tackle nonconvexity is that we will solve the factorized problem~\eqref{eq:problem-bm} at \emph{increasing ranks}, like stepping up a ``staircase'' (\cf Fig.~\ref{fig:pipeline} bottom right). We will start with the lowest rank $r=3$ and solve problem~\eqref{eq:problem-bm} using local optimization. One of two cases will happen. (a) The local optimizer is the same as the global optimizer of the SDP, in which case we can leverage the dual SDP~\eqref{eq:problem-sdp-dual} to certify global optimality and declare victory against the SDP. (b) The local optimizer is not the same as the global optimizer of the SDP, in which case we can again leverage the dual SDP~\eqref{eq:problem-sdp-dual} to ``escape'' the bad local minimum via increasing the rank, i.e., going up the staircase. 

We formalize this in Algorithm~\ref{alg:bm} and prove its correctness.

\setlength{\intextsep}{0pt}  
\begin{algorithm}[ht]
    \caption{Riemannian Staircase}
    \label{alg:bm}
    \begin{algorithmic}[1]
        \STATE \textbf{Input:} the data matrix $Q$
        \STATE \textbf{Output:} optimal solution $U^\star$
        \STATE \texttt{\# Initialization}
        \STATE Set $r = 3$, $U_r^0 = [\eye_3, \dots, \eye_3]$ 
        \WHILE{True}
        \STATE \texttt{\# Local optimization of~\eqref{eq:problem-bm}}
        \STATE $U_r^\star= \textsc{LocalOptimizer}(U_r^0)$ \label{line:local_optimize}
        \STATE \texttt{\# Compute dual certificate}
        \STATE $y_r, Z(y_r) \leftarrow \textsc{Solve}~\eqref{eq:bm_first_order_optimality}$ \label{line:bm_first_order}

        \STATE \texttt{\# Certify global optimality}

        \IF{$Z(y_r) \succeq 0$}
        \STATE \textbf{return} $U_r$ \label{line:bm_return}
        \ENDIF
        \STATE \texttt{\# Escape local minimum}
        \STATE $v \leftarrow \textsc{LeastEigenvector}(Z(y_r))$ \label{line:bm_eigen_vector}
        \STATE $U_{r+1} = \bmat{cc} U_r\tran & \alpha v \emat \tran$ with $\alpha=1$
        \STATE \texttt{\# Line search}
        \WHILE{$\trace{QU_{r+1}U_{r+1}\tran} \geq \trace{QU_rU_r\tran}$}
        \STATE $\alpha = \alpha/2$ \label{line:linesearch-1}
        \STATE $U_{r+1} = \bmat{cc} U_r\tran & \alpha v \emat \tran$ \label{line:linesearch-2}
        \ENDWHILE
        \STATE $r \leftarrow r+1$, $U_{r+1}^0 \leftarrow U_{r+1}$ \label{line:bm_increase_rank}
        \ENDWHILE
    \end{algorithmic}
\end{algorithm}

\textbf{Certify global optimality}. At every iteration of Algorithm~\ref{alg:bm}, it first computes a local optimizer of the BM factorization problem~\eqref{eq:problem-bm} in line~\ref{line:local_optimize} using an algorithm to be described in \S\ref{sec:manifold}. Denote the local optimizer as $U^\star_r$. To check whether $(U^\star_r)\tran U^\star_r$ is the optimal solution to the convex SDP~\eqref{eq:problem-sdp}, we need to compute the dual optimality certificate.

\begin{theorem}[Dual Optimality Certificate]\label{thm:BM}
    Let $U_r^{\star}$ be a locally optimal solution of \prettyref{eq:problem-bm}. Then, the linear independence constraint qualification (LICQ) must hold at $U^{\star}_r$, i.e.,
\begin{align}
\hspace{-4mm}\nabla_U (\langle A_i, (U^{\star}_r)\tran U^{\star}_r \rangle - b_i) = 2A_i (U^{\star}_r)\tran, \quad i = 1, \ldots, m
\end{align}
are linearly independent. Hence, there must exist a unique dual variable $y^{\star}$ such that
\begin{equation}\label{eq:bm_first_order_optimality}
Z(y^{\star})(U^{\star}_r)\tran = \zero.
\end{equation}
If further,
\bea
Z(y^{\star}) &\succeq& 0,
\eea
then $U^{\star}_r$ is a global optimizer of \prettyref{eq:problem-bm}, and $(U^{\star})\tran U^{\star}$,$y^{\star}, Z(y^{\star})$ are optimal for the SDP \prettyref{eq:problem-sdp-standard} and its dual \prettyref{eq:problem-sdp-dual}.
\end{theorem}
\begin{proof}
    In general, LICQ may not hold at a locally optimal solution of the BM factorization. However, in \prettyref{app:proof_licq}, we show that the unique structure of our \nameshort SDP relaxation leads to guaranteed satisfaction of LICQ. The rest of the Theorem is standard and follows from \cite{burer2003nonlinear} or \cite{yang24book-sdp}.
\end{proof}

Theorem~\ref{thm:BM} provides a simple recipe to certify global optimality by solving the linear system of equations in~\eqref{eq:bm_first_order_optimality} (recall $Z(y^\star)$ is linear in $y^\star$ from~\eqref{eq:Zofy}), forming the dual matrix $Z(y^\star)$, and checking its positive semidefinite-ness.

\textbf{Escape local minimum}. If $Z(y^\star)$ is not PSD, then $(U_r^\star)\tran U_r^\star$ is not optimal for the SDP. In this case, the following theorem states that the eigenvector of $Z(y^\star)$ corresponding to the minimum eigenvalue provides a descent direction.




\begin{theorem}[Descent Direction]\label{thm:BM_decrease}

    Let $U^{\star}_r$ be a local minimizer of problem \prettyref{eq:problem-bm} and $y^{\star}$ be the corresponding dual variable. Suppose $Z(y^\star)$ is not PSD and $v$ is an eigenvector of $Z(y^\star)$ corresponding to a negative eigenvalue. Then, consider the BM factorization \prettyref{eq:problem-bm} at rank $r+1$. The direction
    \bea
    D= \bmat{c} 0 \\ v^T \emat
    \eea
    is a descent direction at the point
    \bea
       \hat{U} = \bmat{c} U^{\star}_r \\ 0 \emat.
    \eea
\end{theorem}
\begin{proof}
    See \cite{burer2003nonlinear} or \cite{yang24book-sdp}.
\end{proof}
Theorem~\ref{thm:BM_decrease} states that if Algorithm~\ref{alg:bm} gets stuck at rank $r$, it can escape the local minimum by increasing the rank.
Since $D$ is a descent direction, we perform line search in lines~\ref{line:linesearch-1}-\ref{line:linesearch-2} until a point with lower objective value is found.

\textbf{Global convergence}. Algorithm~\ref{alg:bm} is guaranteed to converge to the optimal solution of the SDP, as stated below.

\begin{theorem}[Global Convergence]\label{thm:global}
    Algorithm~\ref{alg:bm} finds a globally optimal solution of the SDP pair \eqref{eq:problem-sdp-standard}-\eqref{eq:problem-sdp-dual} regardless of the initialization point $U_r^0$ at $r=3$.
\end{theorem}
\begin{proof}
    We show that, in the worst case where $r=3N$, every local optimizer of the BM factorization \eqref{eq:problem-bm} corresponds to a globally optimal solution of the SDP~\cite[Corollary 8]{journee2010low}. To see this, if the locally optimal solution at $r = 3N$ is rank-deficient, then global optimality is guaranteed by the ``rank deficiency'' Lemma~\cite[Proposition 4]{burer2003nonlinear}. Conversely, if the solution is full-rank, then $Z(y^\star)$ must be zero---hence also positive semidefinite---as a result of solving \eqref{eq:bm_first_order_optimality}, which again implies global optimality of the SDP. 
\end{proof}
This property ensures that our approach is initialization-free. To verify the correctness of the theory, we conducted experiments on the BAL-93 dataset using random initializations, and after 1000 trials, our method achieved a 100\% success rate.

It is worth noting that, in the worst case when $r$ is increased to $3N$, the BM factorization does not have any scalability advantage over the original SDP. Fortunately, in almost all numerical experiments, Algorithm~\ref{alg:bm} terminates when $r=3$ or $4$, bringing significant scalability advantage to solving large-scale SDP relaxations.


\subsection{C++/CUDA-based Riemannian Optimization}
\label{sec:manifold}

Everything in Algorithm~\ref{alg:bm} is clear except line~\ref{line:local_optimize} where one needs to locally optimize the nonconvex BM factorization~\eqref{eq:problem-bm}.

\textbf{Riemannian structure}. At first glance, problem~\eqref{eq:problem-bm} looks like a nonconvex constrained optimization problem. However, the constraints of~\eqref{eq:problem-bm} indeed define a smooth manifold. 

\begin{proposition}[BM Riemannian Optimization]
    The BM factorization problem~\eqref{eq:problem-bm} is equivalent to the following \emph{unconstrained} Riemannian optimization problem
    \begin{subequations}\label{eq:BM-Riemannian}
        \bea
        \min  & \trace{ Q U\tran U } \\
\subject & U = \bmat{cccc}R_1 & s_2 R_2 & \cdots & s_N R_N\emat\\
 &s_i \in \mathcal{M}_p, i = 2, \ldots, N,\\
& R_i \in \mathcal{M}_s^{(r)}, i = 1, \ldots, N
        \eea  
    \end{subequations}
    where $\calM_p$ is the positive manifold defined as 
    \bea 
    \calM_p : = \{ s \mid s > 0\},
    \eea
    and $\calM_s^{(r)}$ is the Stiefel manifold of order $r$:
    \bea  
    \mathcal{M}_s^{(r)} = \{R\in\Real{r\times3} \mid R\tran R = \eye_3,\} .
    \eea
\end{proposition}
\begin{proof}
    By inspection.
\end{proof}

We solve problem~\eqref{eq:BM-Riemannian} using the Riemannian trust-region algorithm with truncated conjugate gradient (Rtr-tCG). Details of this algorithm can be found in \cite{absil2008optimization, boumal2023introduction}. 




\textbf{C++/CUDA implementation}. The Rtr-tCG algorithm is readily available through the \manopt package~\cite{boumal2014manopt}. However, to boost efficiency and enable fast solution of the SDP relaxation, we implement the Rtr-tCG algorithm directly in C++/CUDA, using analytical hessian-vector product.


\begin{itemize}
    \item \textbf{Conjugate gradient method}.
    The conjugate gradient method involves only Hessian-vector products and vector addition. The Hessian-vector product can be decomposed into two components: (a) the Euclidean Hessian-vector product and (b) the Riemannian projection onto the tangent space. The first component primarily requires matrix-vector multiplication, which can be efficiently implemented using \texttt{cuBLAS}. The second component involves batched small matrix-matrix multiplications, matrix-matrix inner products, scalar-matrix multiplications, all of which are implemented using custom CUDA kernels. For vector addition, we directly utilize the \texttt{cublasDaxpy} function from \texttt{cuBLAS}.

    \item \textbf{Retraction}. The retraction operation aims to map a point from the tangent space back to the manifold. In our case, the retraction operation happens both on the positive manifold and the Stiefel manifold. The retraction on the positive manifold is a simple custom kernel, while the retraction on the Stiefel manifold involves QR decomposition. We directly apply Gram-Schmidt process on every batch of $3 \times r$ matrices, which is implemented using custom CUDA kernels.
\end{itemize}

In \S\ref{sec:exp}, we show our GPU-based implementation achieves up to $100$ times speedup compared to the CPU-based \manopt.

\begin{tcolorbox}
    \begin{center}
        \vspace{-1mm}
        \textbf{Summary}
        \vspace{-1mm}
    \end{center} 
    We focused on developing a customized solver capable of solving large-scale SDP relaxations in~\eqref{eq:problem-sdp} (and~\eqref{eq:problem-sdp-standard}). We applied the Burer-Monteiro factorization method to exploit low-rankness of the optimal SDP solutions (\cf problem~\eqref{eq:problem-bm}), and leveraged the Staircase Algorithm~\ref{alg:bm} to solve the nonconvex BM factorization to global optimality. We pointed out the BM factorization problem is indeed an unconstrained Riemannian optimization problem (\cf problem~\eqref{eq:BM-Riemannian}) and developed a C++/CUDA-based implementation that is significantly faster than \manopt.
\end{tcolorbox}

\begin{remark}[Connection to Prior Work]
    This is not the first time BM factorization has been applied in robotics. \sesync and related works \cite{dellaert2020shonan,rosen2021scalable} pioneered the application of BM factorization for solving the pose graph optimization problem. Several works~\cite{garcia2021certifiable,garcia2024certifiable,holmes2023efficient,holmes2024sdprlayers} utilized the dual optimality certificate result in Theorem~\ref{thm:BM} to develop fast certifiers. Our novelty lies in developing the first C++/CUDA-based implementation of the Riemannian trust-region algorithm to push the scalability limitations of BM factorization.
\end{remark}

\begin{figure*}[!t]
    \centering
    \includegraphics[width=0.9\linewidth]{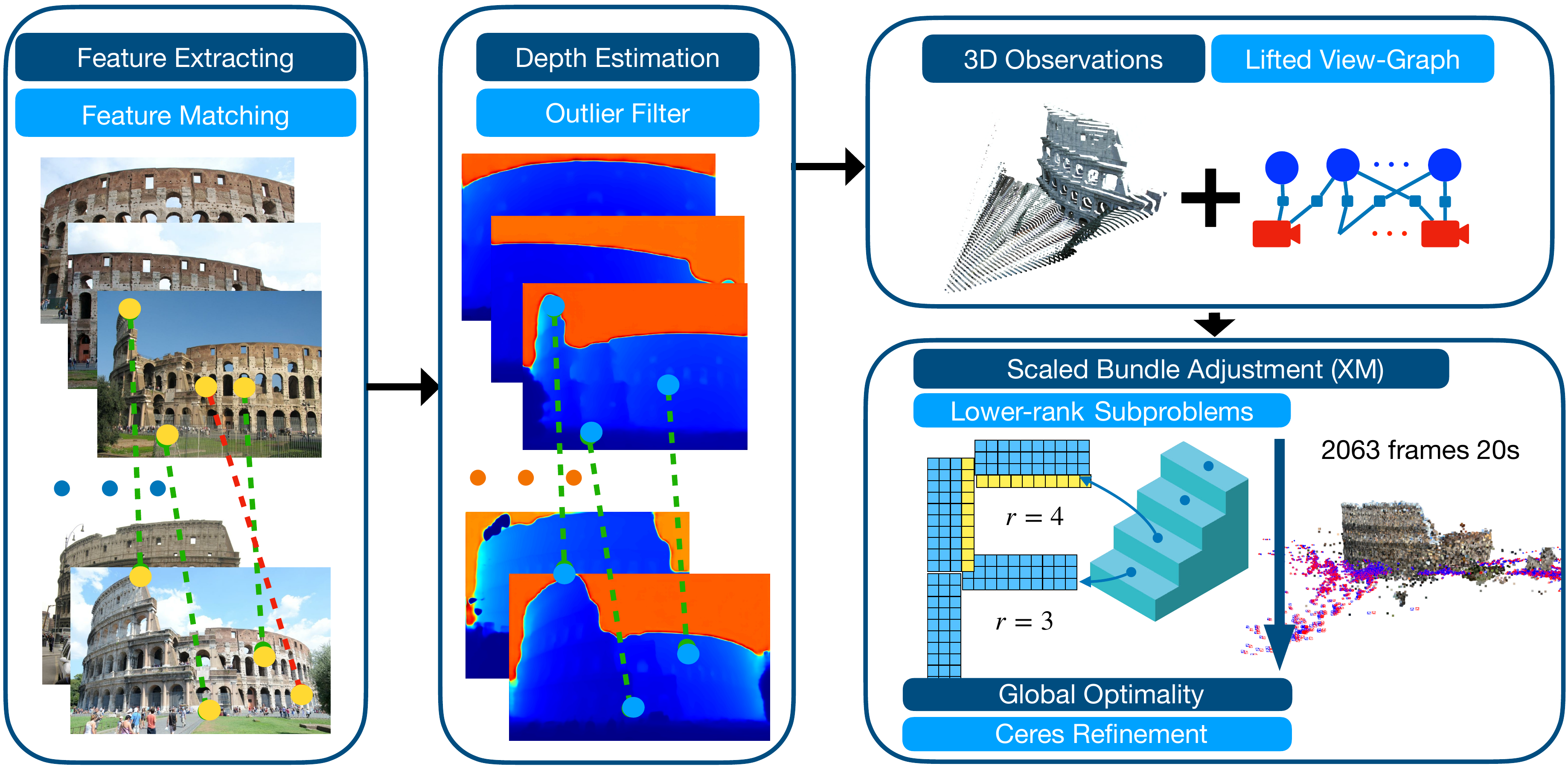}
    \caption{\xmsfm: structure from motion pipeline with \nameshort.}
    \label{fig:pipeline}
    \vspace{-3mm}
\end{figure*}

\vspace{-3mm}
\section{Structure from Motion with XM}
\label{sec:xm-sfm}

In this section, we present our SfM pipeline with \nameshort as the optimization engine, illustrated in Fig.~\ref{fig:pipeline}.

{\bf View graph}.
To construct view graph for an image set, we first run \colmap's \textit{feature extractor} and \textit{exhaustive matcher} to extract 2D correspondences. The {feature extractor} employs SIFT \cite{lowe1999sift} for feature detection and description, while the {exhaustive matcher} matches every image pair. After matching, we apply \glomap's \textit{track establishment} to produce a four-column file where the first two columns represent feature point coordinates, the third indicates the image index, and the fourth corresponds to the 3D landmark indices. Currently, we use the original implementation from \colmap and \glomap. However, we remark that it is possible to speed up the processes further using C++ and GPU implementation. We leave this improvement for future work.


{\bf Depth estimation}.
We use the depth estimation model \textsc{Unidepth}~\cite{piccinelli2024unidepth} to calculate the metric depth of a given image, and lift the view graph from 2D to 3D. It is worth emphasizing that our test datasets do \textbf{NOT} belong to the training dataset of \textsc{Unidepth}. If given the confidence map, we use it to update the weight of different observations.
We also tried other depth prediction models in \prettyref{app:depthcomparison}.

{\bf Filter from two-view estimation}.
Using 2D observations, we estimate the relative pose between two images. Based on this pose, we filter out 3D landmarks with large Euclidean distance errors. Specifically, landmarks with distance errors exceeding three times the median are removed.

{\bf \nameshort solver}.
We then use the lifted 3D measurements and the view-graph to form a $Q$ matrix as shown in \prettyref{eq:problem_scale_rotation_only}. We solve the \sdp~problem in \prettyref{eq:problem-sdp} using our \nameshort solver. If needed, we also delete the 10\% measurements with the largest residuals and re-run the \nameshort solver. This corresponds to a greedy heuristic for outlier removal~\cite{antonante2021outlier,sim2006removing} and we call it \xmdouble (running twice).

{\bf \ceres refinement}.
Usually the depth predictions are quite noisy, leading to inaccurate estimations of \nameshort. We therefore also feed the estimated poses and landmarks to \ceres as a warmstart to solve the original bundle adjustment problem~\eqref{eq:ba-colmap}. As we will show, the solution of \nameshort always provide a strong warmstart for \ceres to quickly optimize~\eqref{eq:ba-colmap}.


\section{Experiments}
\label{sec:exp}
\begin{table*}[t]
    \centering
    \caption{\textnormal{Results on the BAL dataset. We report the ATE, RPE and running time for \ceres, \manopt, and our proposed \nameshort solver. The evaluation is conducted on four BAL datasets with varying numbers of frames to demonstrate that our method is both fast and accurate across datasets of different scales (e.g., BAL-10155 indicates there are 10155 camera frames to be reconstructed). \ceres fails at a bad local minimum without a good initial guess, while \manopt performs significantly slower and even fails to solve within ten hours on the largest dataset. 
    }}
    \label{tab:bal}
    \begin{adjustbox}{width=\linewidth}
        \begin{tabular}{l|c|c|c|c|c|c|c|c|c|c|c|c|}
            \toprule
            Datasets & \multicolumn{3}{c|}{BAL-93\,(6033\,landmarks)} & \multicolumn{3}{c|}{BAL-392\,(13902\,landmarks)} & \multicolumn{3}{c|}{BAL-1934\,(67594\,landmarks)} & \multicolumn{3}{c|}{BAL-10155\,(33782\,landmarks)} \\
            \midrule
            Metrics & \stackon{Solver\,Time}{Processing\,Time} & \stackon{ATE-T}{ATE-R} & \stackon{RPE-T}{RPE-R} &  \stackon{Solver\,Time}{Processing\,Time} & \stackon{ATE-T}{ATE-R} & \stackon{RPE-T}{RPE-R} & \stackon{Solver\,Time}{Processing\,Time} & \stackon{ATE-T}{ATE-R} & \stackon{RPE-T}{RPE-R} & \stackon{Solver\,Time}{Processing\,Time} & \stackon{ATE-T}{ATE-R} & \stackon{RPE-T}{RPE-R}  \\
            \midrule\midrule
            \textsc{Ceres} & \stackon{$4.21$}{$\bf 0.06$} & \stackon{$0.18$}{$19.47^{\circ}$} & \stackon{$0.29$}{$15.46^{\circ}$} & \stackon{$23.31$}{$0.37$} & \stackon{$0.43$}{$23.49^{\circ}$} & \stackon{$0.7$}{$21.80^{\circ}$} & \stackon{$227.15$}{$3.85$} & \stackon{$0.32$}{$177.20^{\circ}$} & \stackon{$0.63$}{$31.50^{\circ}$} & \stackon{$50.97$}{$2.41$} & \stackon{$0.71$}{$84.29^{\circ}$} & \stackon{$1.25$}{$80.19^{\circ}$} \\
            \midrule
            \textsc{Ceres-gt} & \stackon{$0.36$}{$0.06$} & \stackon{$\bf0.0$}{$\bf0.0^{\circ}$} & \stackon{$0.01$}{$\bf0.0^{\circ}$} & \stackon{$3.59$}{$\bf0.36$} & \stackon{$0.01$}{$\bf0.0^{\circ}$} & \stackon{$0.01$}{$\bf0.0^{\circ}$} & \stackon{$16.39$}{$3.78$} & \stackon{$0.01$}{$\bf0.0^{\circ}$} & \stackon{$\bf0.01$}{$\bf0.0^{\circ}$} & \stackon{$\bf21.49$}{$\bf2.4$} & \stackon{$\bf0.0$}{$\bf0.0^{\circ}$} & \stackon{$\bf0.0$}{$\bf0.0^{\circ}$} \\
            \midrule
            \textsc{Ceres-gt-0.01} & \stackon{$11.61$}{$0.06$} & \stackon{$0.01$}{$1.72^{\circ}$} & \stackon{$0.02$}{$2.87^{\circ}$} & \stackon{$25.74$}{$0.38$} & \stackon{$0.04$}{$4.58^{\circ}$} & \stackon{$0.07$}{$5.73^{\circ}$} & \stackon{$427.79$}{$\bf3.73$} & \stackon{$0.01$}{$\bf0.0^{\circ}$} & \stackon{$\bf0.01$}{$\bf0.0^{\circ}$} & \stackon{$380.4$}{$2.45$} & \stackon{$0.04$}{$2.87^{\circ}$} & \stackon{$0.07$}{$5.15^{\circ}$} \\
            \midrule
            \textsc{Ceres-gt-0.1} & \stackon{$9.24$}{$0.06$} & \stackon{$0.56$}{$40.29^{\circ}$} & \stackon{$1.01$}{$34.38^{\circ}$} & \stackon{$21.34$}{$0.37$} & \stackon{$0.37$}{$22.37^{\circ}$} & \stackon{$0.73$}{$30.38^{\circ}$} & \stackon{$70.43$}{$3.76$} & \stackon{$0.82$}{$\bf0.0^{\circ}$} & \stackon{$1.3$}{$9.17^{\circ}$} & \stackon{$35.06$}{$2.38$} & \stackon{$0.33$}{$20.57^{\circ}$} & \stackon{$0.67$}{$37.78^{\circ}$} \\
            \midrule\midrule
            \textsc{Manopt} & \stackon{$0.56$}{$0.94$} & \stackon{$\bf0.0$}{$\bf0.0^{\circ}$} & \stackon{$\bf0.0$}{$\bf0.0^{\circ}$} & \stackon{$21.43$}{$2.7$} & \stackon{$\bf0.0$}{$\bf0.0^{\circ}$} & \stackon{$\bf0.0$}{$\bf0.0^{\circ}$} & \stackon{$236.46$}{$69.15$} & \stackon{$\bf0.0$}{$\bf0.0^{\circ}$} & \stackon{$\bf0.01$}{$\bf0.0^{\circ}$} & \stackon{$**$}{$336.76$} & \stackon{$**$}{$**$} & \stackon{$**$}{$**$} \\
            \midrule
            \textsc{\nameshort} & \stackon{$\bf0.07$}{$0.94$} & \stackon{$\bf0.0$}{$\bf0.0^{\circ}$} & \stackon{$\bf0.0$}{$\bf0.0^{\circ}$} & \stackon{$\bf0.55$}{$2.7$} & \stackon{$\bf0.0$}{$\bf0.0^{\circ}$} & \stackon{$\bf0.0$}{$\bf0.0^{\circ}$} & \stackon{$\bf2.09$}{$69.15$} & \stackon{$\bf0.0$}{$\bf0.0^{\circ}$} & \stackon{$\bf0.01$}{$\bf0.0^{\circ}$} & \stackon{$ 4322.77 $}{$ 336.76$} & \stackon{$ 0.02 $}{$ 0.73 ^{\circ}$} & \stackon{$ 0.02 $}{$ 0.5 ^{\circ}$} \\
            \bottomrule
        \end{tabular}
    \end{adjustbox}
    \vspace{-4mm}
\end{table*}

\begin{table}[t]
    \centering
    \caption{\textnormal{Results on the BAL dataset (suboptimality and min-eig). We report the suboptimality gap and minimum eigenvalue of $Z$ matrix in \prettyref{eq:problem-sdp-dual} for our proposed \nameshort method.
    }}
    \label{tab:bal-xm}
    \begin{adjustbox}{width=\linewidth}
    \begin{tabular}{l|c|c|c|c|}
        \toprule
        Datasets & \multicolumn{1}{c|}{BAL-93} & \multicolumn{1}{c|}{BAL-392} & \multicolumn{1}{c|}{BAL-1934} & \multicolumn{1}{c|}{BAL-10155} \\
        \midrule
        Suboptimality-Gap & $4.8 \times 10^{-4}$ & $4.2 \times 10^{-3}$ & $1.0 \times 10^{-2}$ & $6.2 \times 10^{-1}$ \\
        \midrule
        Min-eig & $-8.8 \times 10^{-5}$ & $1.3 \times 10^{-7}$ & $1.2 \times 10^{-6}$ & $-2.1 \times 10^{1}$ \\

        \bottomrule
    \end{tabular}
\end{adjustbox}
\vspace{-3mm}
\end{table}
\begin{figure*}[!t]
    \includegraphics[width=\linewidth]{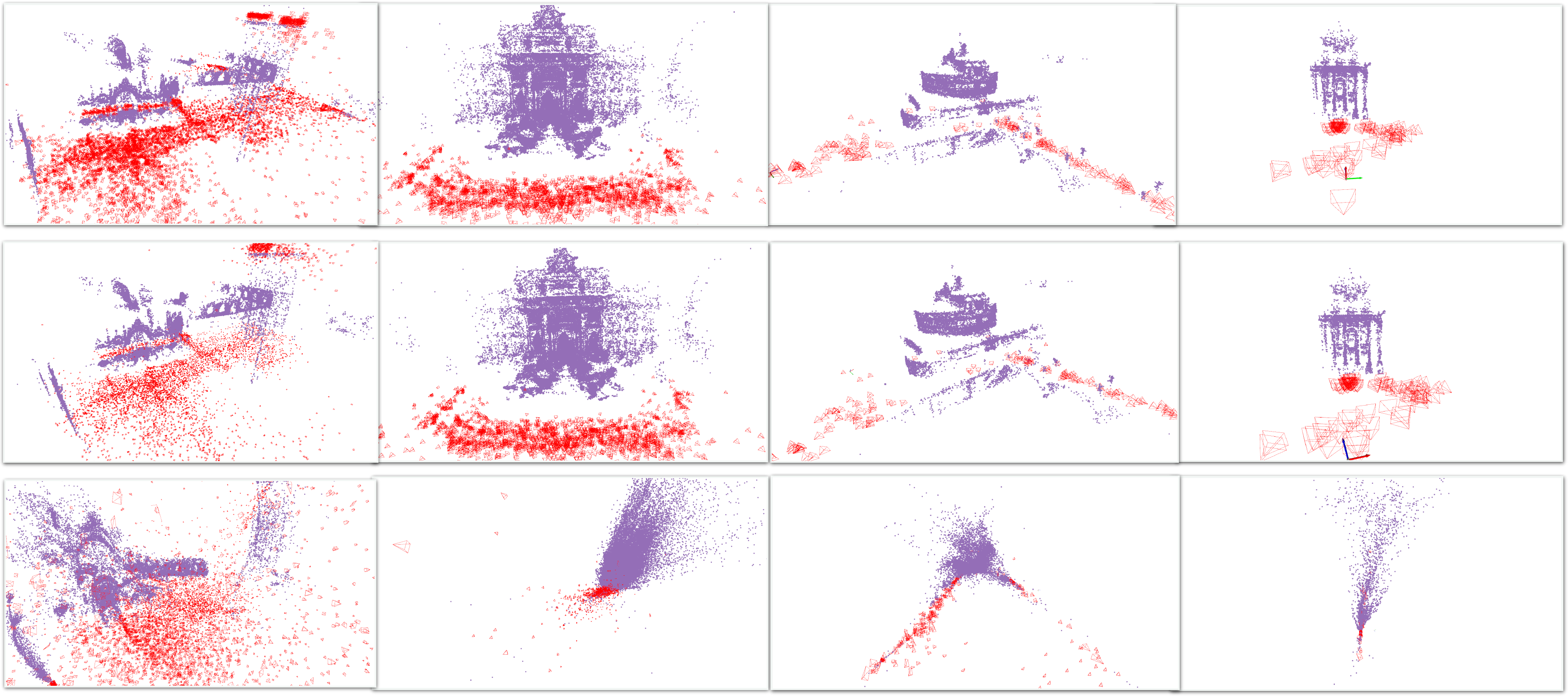}
    \caption{Visualization of BAL datasets. \textbf{Top:} Our \nameshort solver. \textbf{Middle:} \textsc{Ceres-GT-0.01}. \textbf{Bottom:} \textsc{Ceres-GT-0.1}. Both our \nameshort solver and \textsc{Ceres-GT-0.01} accurately recover the ground truth camera poses and landmarks, whereas \textsc{Ceres-GT-0.1} fails.}
    \label{fig:visualization-BAL}
    \vspace{-2mm}
\end{figure*}

\begin{figure*}[!t]
    \includegraphics[width=\linewidth]{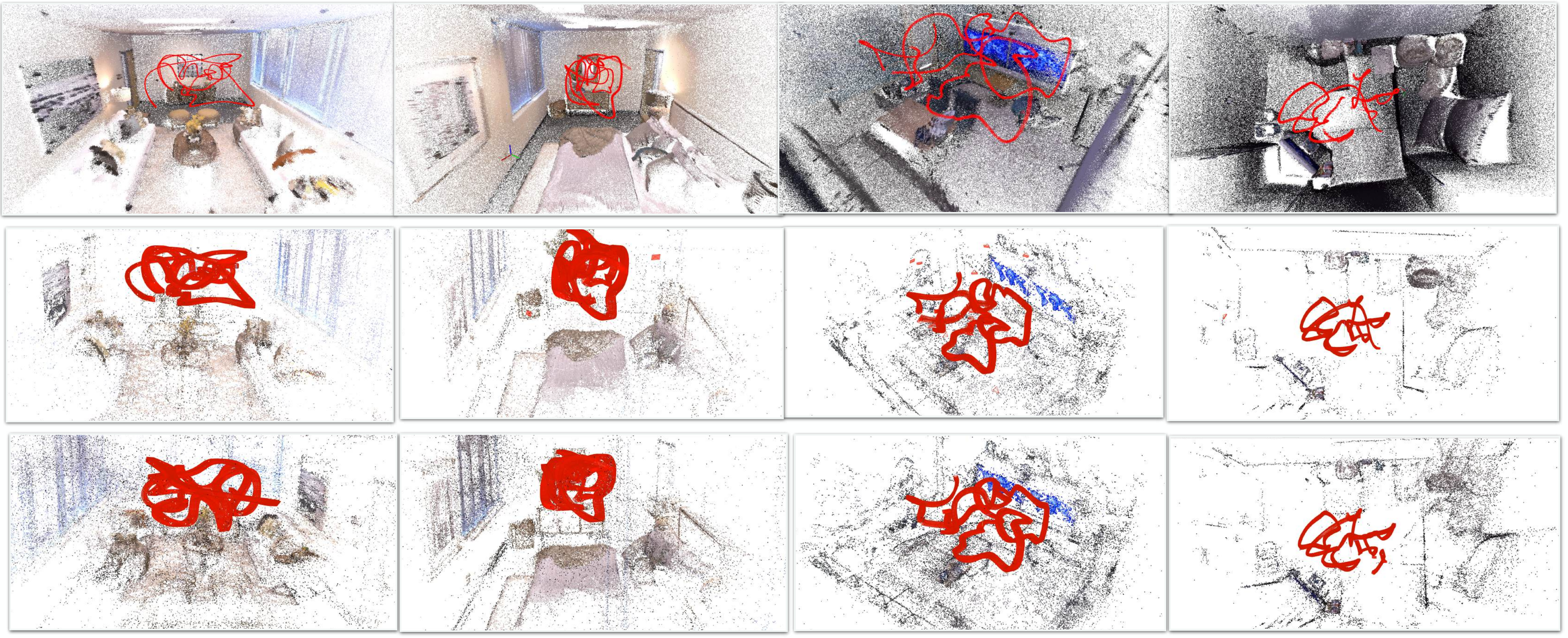}
    \caption{Visualization of Replica datasets. \textbf{Top:} Our \nameshort solver. \textbf{Middle:} \glomap. \textbf{Bottom:} \colmap. All methods achieve high accuracy, producing nearly identical reconstruction results. \glomap sometimes produce outliers (see column 2 and 3).
    }
    \label{fig:visualization-replica}
    \vspace{-5mm}
\end{figure*}

\begin{figure*}[!t]
    \includegraphics[width=\linewidth]{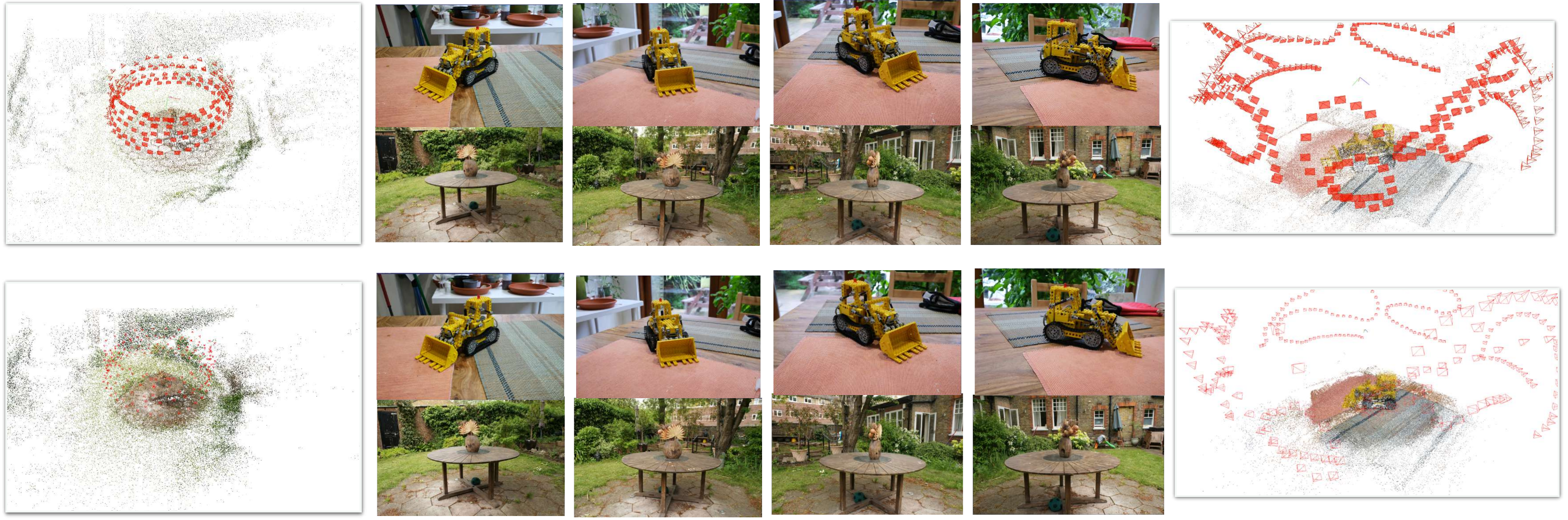}
    \caption{Visualization of Mip-Nerf datasets. \textbf{Top:} \colmap. \textbf{Bottom:} Our \nameshort solver. 3D-gaussian renderings are the same.}
    \label{fig:visualization-mipnerf}
    \vspace{-0mm}
\end{figure*}

\begin{figure*}[!t]
    \includegraphics[width=\linewidth]{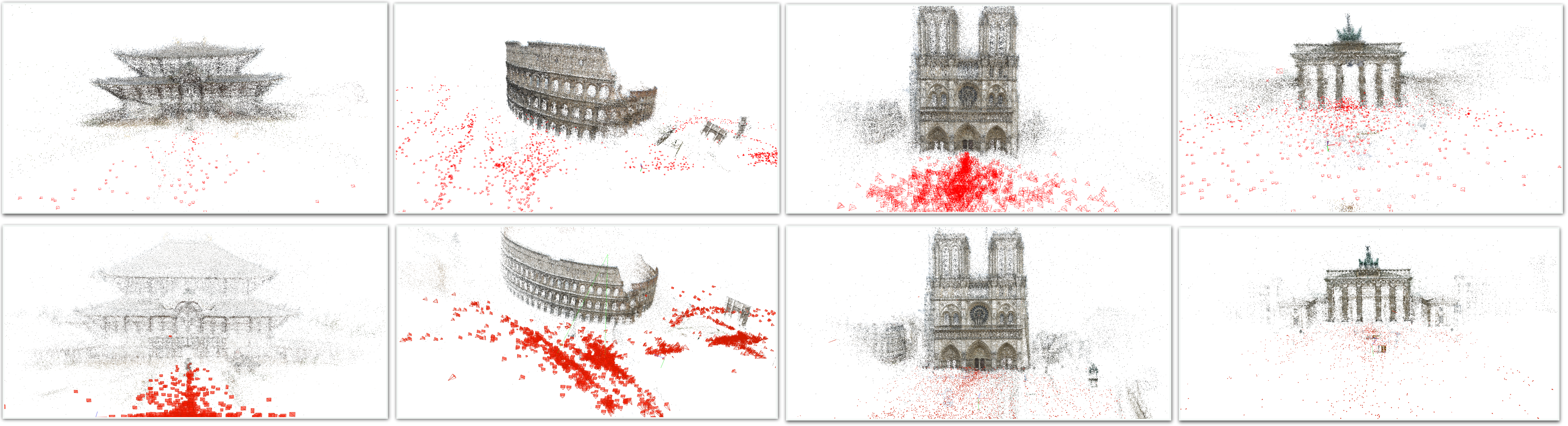}
    \caption{Visualization of IMC2023 datasets. \textbf{Top:} Our \nameshort solver. \textbf{Bottom:} \glomap.}
    \label{fig:visualization-imc}
    \vspace{-0mm}
\end{figure*}

\begin{figure*}[!t]
    \includegraphics[width=\linewidth]{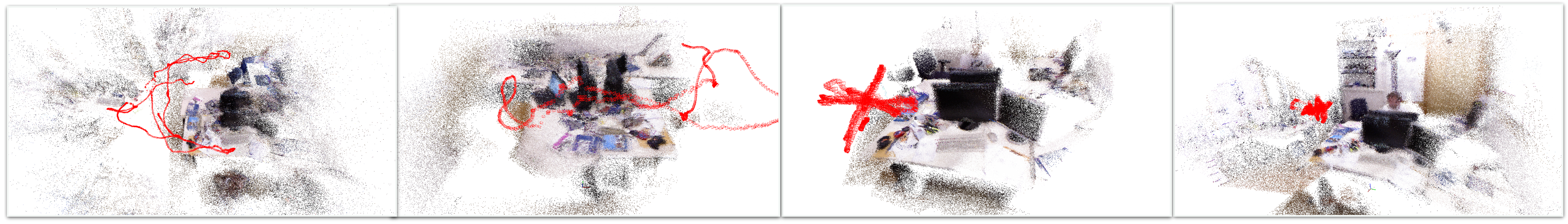}
    \caption{Visualization of TUM datasets using our \nameshort solver.}
    \label{fig:visualization-tum}
    \vspace{-0mm}
\end{figure*}

\begin{figure*}[!t]
    \includegraphics[width=\linewidth]{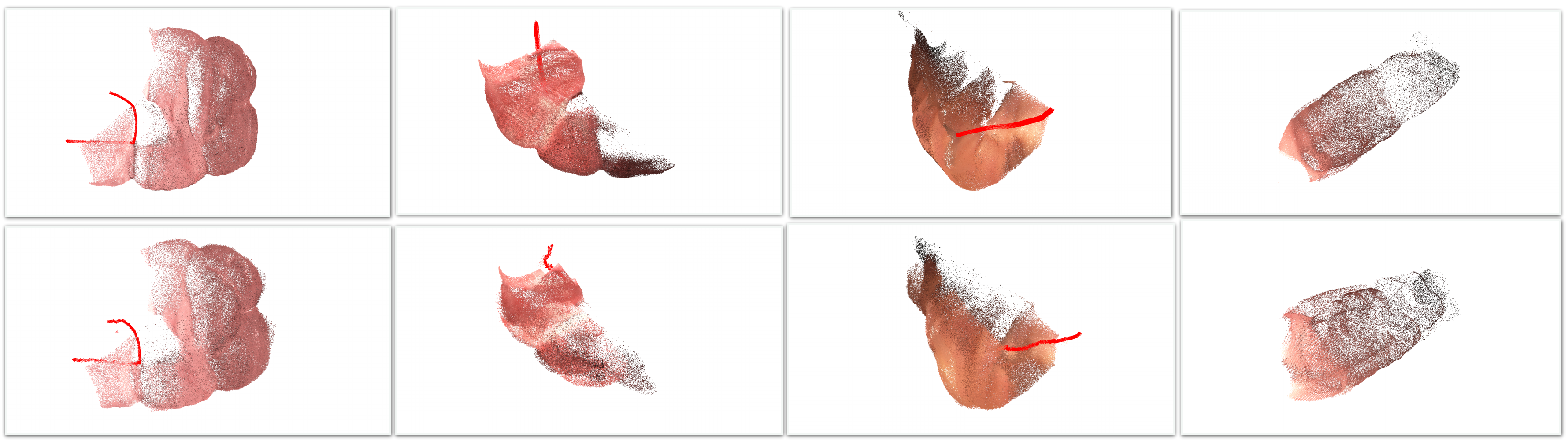}
    \caption{Visualization of C3VD medical datasets. \textbf{Top:}  With ground truth depth. \textbf{Bottom:} With learned depth.}
    \label{fig:visualization-c3vd}
    \vspace{-0mm}
\end{figure*}

We evaluate the \nameshort solver and the \xmsfm pipeline on diverse datasets. \nameshort is benchmarked against the leading bundle adjustment solver \ceres and \xmsfm is compared against the widely adopted SfM pipelines \colmap and \glomap. We further analyze the convergence rate and evaluate robustness to depth estimation noise in \prettyref{app:converge-noise}.

Experiments run on a Lambda Vector workstation with 64-core AMD® Ryzen Threadripper Pro 5975WX CPUs and dual NVIDIA® RTX 6000 GPUs, using CUDA 12.4 and Python 3.11.10. All dependencies are carefully set up to leverage multi-CPU and GPU acceleration.




\subsection{BAL Dataset}\label{sec:exp-bal}

We first evaluate on the Bundle Adjustment in the Large (BAL) dataset \cite{agarwal2010bundle}. This dataset contains reconstruction results from Flickr photographs using Bundler. On BAL we focus on \nameshort  solver performance rather than the full \xmsfm pipeline. 

\textbf{Setup and baselines}.
For input, we use 2D observations for \ceres and 3D observations for \nameshort. The 2D keypoint measurements come directly from the BAL dataset, while the 3D keypoint measurements are lifted by appending $z$-coordinates to the 2D measurements. Though accurate, these 3D observations incorporate slight noise, making them a refined yet imperfect ground truth. To showcase \nameshort's efficiency in solving the SDP relaxation of the SBA problem~\eqref{eq:sba}, we also evaluate \manopt using the same input as \nameshort.


\textbf{Metrics}.
We evaluate performance based on the runtime and the median of Absolute Trajectory Error (ATE) and Relative Pose Error (RPE). More details can be found in \prettyref{app:metrics}. Additionally, we report the suboptimality and the minimum eigenvalue of the $Z(y)$ matrix in \prettyref{eq:problem-sdp-dual} to demonstrate that \nameshort achieves global optimality. Note that the minimum eigenvalue of $Z(y)$ has been used as a metric for global optimality in previous works as well~\cite{carlone2015lagrangian}. Runtime is split into preprocessing time and solver time. The former primarily involves constructing the $Q$ matrix in \prettyref{eq:problem_scale_rotation_only}, which is implemented in Python, while the latter corresponds to the GPU solver. We separate preprocessing time and solver time because there are still ways to further reduce the preprocessing time (while the solver time, to the best of our understanding, has been pushed to the limit). For example, as building the $Q$ matrix requires large dense matrix multiplications, a GPU implementation is expected to accelerate this step by 10 to 100 times. However, we leave this as a future step.

\textbf{Results}.
Table~\ref{tab:bal} summarizes the comparison between \nameshort and other methods. We tested several versions of \ceres. ``\ceres'' indicates running \ceres without any initialization. ``\textsc{Ceres-GT}'' indicates starting \ceres at the groundtruth estimation. ``\textsc{Ceres-GT-0.01}'' means adding noise to the groundtruth with standard deviation 0.01. We make several observations. (a) Without good initialization, \ceres does not work, as shown by the failures of \ceres and \textsc{Ceres-GT-0.1}. (b) \nameshort is up to $100$ times faster than \manopt, showing the superior efficiency of our GPU implementation. Notably, \nameshort's solver time is below a second for $N$ in the order of hundreds, and \nameshort scales to $N > 10,000$ camera frames. 

Table~\ref{tab:bal-xm} presents the suboptimality gap and minimum eigenvalue of \nameshort. As we can see, except the largest instance with $N=10155$ camera frames, \nameshort solved all the other instances to certifiable global optimality. The reason why \nameshort did not solve the largest instance to global optimality is because we restricted its runtime to one hour (\nameshort indeed achieve global optimality if allowed four hours of runtime). 

Fig.~\ref{fig:visualization-BAL} visualizes the 3D reconstructions. Since real images are not available in BAL, all 3D landmarks have the same purple color. The reconstructed cameras are shown in red.

\begin{tcolorbox}
    \begin{center}
        \vspace{-1mm}
        \textbf{Takeaway}
        \vspace{-2mm}
    \end{center} 
    $\bullet$ The SBA problem~\eqref{eq:sba} is easier to solve than the BA problem~\eqref{eq:ba-colmap}---we can design efficient convex relaxations. While \ceres needs good initialization for solving~\eqref{eq:ba-colmap}, \nameshort requires no initializations for solving~\eqref{eq:sba}.
        
    $\bullet$ With GT depth and outlier-free matchings, solving SBA with \nameshort produces the same result as solving BA with \ceres or \colmap.
       
    $\bullet$ \nameshort is fast and scalable.
\end{tcolorbox}




\subsection{Replica Dataset}\label{sec:exp-replica}

\begin{table*}[t]
    \vspace{-5mm}
    \centering
    \caption{\textnormal{Results on the Replica dataset.  \colmap, while highly stable, is extremely slow, taking over 20 hours for datasets with 2000 frames. \glomap improves speed but still requires several hours for the solving stage and occasionally produces outliers. In comparison, our solver achieves similar accuracy in just 10 seconds for the same dataset size.}}
    \label{tab:replica}
  
    \begin{adjustbox}{width=\linewidth}
        \begin{tabular}{l|c|c|c|c|c|c|c|c|c|c|c|c|}
            \toprule
            Datasets & \multicolumn{3}{c|}{\textsc{\xmdouble}} & \multicolumn{3}{c|}{\textsc{Filter + \xmdouble}} & \multicolumn{3}{c|}{\glomap} & \multicolumn{3}{c|}{\colmap} \\
            \midrule
            Metrics & \stackon{Solver\,Time}{Processing\,Time} & \stackon{ATE-T}{ATE-R} & \stackon{RPE-T}{RPE-R} & \stackon{Solver\,Time}{Processing\,Time} & \stackon{ATE-T}{ATE-R} & \stackon{RPE-T}{RPE-R} & \stackon{Solver\,Time}{Processing\,Time} & \stackon{ATE-T}{ATE-R} & \stackon{RPE-T}{RPE-R} & \stackon{Solver\,Time}{Processing\,Time} & \stackon{ATE-T}{ATE-R} & \stackon{RPE-T}{RPE-R} \\
            \midrule\midrule
            Room0-100 
            & \stackon{$\bf 0.12 $}{$ 10.52 $} & \stackon{$ 0.005 $}{$\bf 0.323 ^{\circ}$} & \stackon{$ 0.007 $}{$ 0.083 ^{\circ}$}
            & \stackon{$\bf 0.12 $}{$ 13.71 $} & \stackon{$ 0.005 $}{$ 1.025 ^{\circ}$} & \stackon{$ 0.008 $}{$ 0.077 ^{\circ}$}
            & \stackon{$ 39.57 $}{$ 6.4 $} & \stackon{$\bf 0.003 $}{$ 0.469 ^{\circ}$} & \stackon{$\bf 0.004 $}{$ 0.05 ^{\circ}$} 
            & \stackon{$ 100.61 $}{$\bf 1.71 $}& \stackon{$\bf 0.003 $}{$ 0.582 ^{\circ}$} & \stackon{$\bf 0.004 $}{$\bf 0.038 ^{\circ}$}\\
            \midrule
            Room0-2000
            & \stackon{$\bf 8.5 $}{$ 1142.74 $} & \stackon{$ 0.004 $}{$ 1.38 ^{\circ}$} & \stackon{$ 0.006 $}{$ 0.44 ^{\circ}$}
            & \stackon{$ 16.26 $}{$ 1868.24 $} & \stackon{$ 0.003 $}{$ 0.379 ^{\circ}$} & \stackon{$ 0.005 $}{$ 0.242 ^{\circ}$}
            & \stackon{$ 5371.98 $}{$ 957.31 $} & \stackon{$\bf 0.001 $}{$ 0.136 ^{\circ}$} & \stackon{ $\bf0.001 $}{$ 0.043 ^{\circ}$} 
            & \stackon{$ 74827.56 $}{$\bf 120.02 $} & \stackon{$\bf 0.001 $}{$\bf 0.078 ^{\circ}$} & \stackon{$\bf 0.001 $}{$\bf 0.068 ^{\circ}$}\\
            \midrule\midrule
            Room1-100 
            & \stackon{$\bf 0.14 $}{$ 13.08 $} & \stackon{$\bf 0.009 $}{$ 1.847 ^{\circ}$} & \stackon{$\bf 0.013 $}{$ 0.16 ^{\circ}$}
            & \stackon{$ 0.14 $}{$ 16.19 $} & \stackon{$ 0.01 $}{$ 2.056 ^{\circ}$} & \stackon{$ 0.014 $}{$ 0.176 ^{\circ}$}
            & \stackon{$ 53.95 $}{$ 8.75 $} & \stackon{$ 0.014 $}{$\bf 1.409 ^{\circ}$} & \stackon{$ 0.019 $}{$\bf 0.122 ^{\circ}$} 
            & \stackon{$ 144.39 $}{$\bf 1.99 $} & \stackon{$ 0.06 $}{$ 10.765 ^{\circ}$} & \stackon{$ 0.072 $}{$ 0.726 ^{\circ}$}\\
            \midrule
            Room1-2000
            & \stackon{$ 7.37 $}{$ 1156.6 $} & \stackon{$ 0.006$}{$\bf 0.005 ^{\circ}$} & \stackon{$ 0.009 $}{$\bf 0.018 ^{\circ}$}
            & \stackon{$\bf 6.18 $}{$ 1651.12 $} & \stackon{$ 0.005 $}{$ 0.428 ^{\circ}$} & \stackon{$ 0.008 $}{$ 0.418 ^{\circ}$}
            & \stackon{$ 3806.94 $}{$ 911.0 $} & \stackon{$\bf 0.001 $}{$ 0.166 ^{\circ}$} & \stackon{$\bf 0.002 $}{$ 0.066 ^{\circ}$} 
            & \stackon{$ 64912.5 $}{$\bf 379.45 $} & \stackon{$ 0.002 $}{$0.156 ^{\circ}$} & \stackon{$\bf 0.002 $}{$ 0.06 ^{\circ}$}\\
            \midrule\midrule
            Office0-100 
            & \stackon{$\bf 0.15 $}{$ 12.75 $} & \stackon{$ 0.04 $}{$ 3.554 ^{\circ}$} & \stackon{$ 0.051 $}{$ 0.128 ^{\circ}$}
            & \stackon{$ 0.12 $}{$ 16.53 $} & \stackon{$ 0.024 $}{$ 3.385 ^{\circ}$} & \stackon{$ 0.033 $}{$ 0.127 ^{\circ}$}
            & \stackon{$ 27.99 $ }{$ 8.67 $}& \stackon{$\bf 0.015 $}{$\bf 0.559 ^{\circ}$} & \stackon{$\bf 0.019 $}{$\bf 0.04 ^{\circ}$}
            & \stackon{$ 49.02 $}{$\bf 3.58 $ }& \stackon{$ 0.019 $}{$ 1.53 ^{\circ}$} & \stackon{$ 0.028 $}{$ 0.067 ^{\circ}$}\\
            \midrule
            Office0-2000
            & \stackon{$\bf 4.6 $}{$ 730.15 $} & \stackon{$ 0.016 $}{$\bf 0.012 ^{\circ}$} & \stackon{$ 0.024 $}{$\bf  0.015 ^{\circ}$}
            & \stackon{$ 5.61 $}{$ 1181.42 $} & \stackon{$ 0.016 $}{$ 0.053 ^{\circ}$} & \stackon{$ 0.019 $}{$ 0.051 ^{\circ}$}
            & \stackon{$ 3241.77 $}{$ 613.85 $} & \stackon{$ 0.024 $}{$ 2.822 ^{\circ}$} & \stackon{$ 0.034 $}{$ 0.058 ^{\circ}$} 
            & \stackon{$ 35682.24 $}{$\bf 168.36 $} & \stackon{$\bf 0.003 $}{$ 0.184 ^{\circ}$} & \stackon{$\bf 0.004 $}{$0.06 ^{\circ}$}\\
            \midrule\midrule
            Office1-100 
            & \stackon{$\bf 0.13 $}{$ 8.13 $} & \stackon{$ 0.022 $}{$ 2.708 ^{\circ}$} & \stackon{$ 0.03 $}{$ 0.2 ^{\circ}$}
            & \stackon{$ 0.15 $}{$ 10.81 $} & \stackon{$ 0.015 $}{$ 2.032 ^{\circ}$} & \stackon{$ 0.023 $}{$ 0.138 ^{\circ}$}
            & \stackon{$ 18.7 $}{$ 4.37 $}& \stackon{$\bf 0.006 $}{$\bf 1.423 ^{\circ}$} & \stackon{$\bf 0.008 $}{$\bf 0.021 ^{\circ}$} 
            & \stackon{$ 61.5 $}{$\bf 1.36$ }& \stackon{$ 0.014 $}{$ 3.366 ^{\circ}$} & \stackon{$ 0.021 $}{$ 0.039 ^{\circ}$}\\
            \midrule
            Office1-2000
            & \stackon{$\bf 23.1 $}{$ 759.43 $} & \stackon{$ 0.077 $}{$ 10.516 ^{\circ}$} & \stackon{$ 0.113 $}{$ 3.802 ^{\circ}$}
            & \stackon{$ 230.25 $}{$ 1270.29 $} & \stackon{$ 0.038 $}{$ 5.08 ^{\circ}$} & \stackon{$ 0.052 $}{$ 1.643 ^{\circ}$}
            & \stackon{$ 2838.54 $}{$ 666.34 $} & \stackon{$\bf 0.001 $}{$\bf 0.037 ^{\circ}$} & \stackon{$\bf 0.001 $}{$\bf 0.045 ^{\circ}$}
            &\stackon{$ 25512.84 $}{$\bf 140.53 $} & \stackon{$ 0.078 $}{$ 13.882 ^{\circ}$} & \stackon{$ 0.112 $}{$ 0.157 ^{\circ}$}\\
            \bottomrule
        \end{tabular}
    \end{adjustbox}
    \vspace{-5mm}
\end{table*}

\begin{table}[t]
    \centering
    \caption{\textnormal{Results on the Replica dataset (min-eig and suboptimality). All the min-eig and suboptimality-gap are small, which means our solver find the global minima.
    }}
    \label{tab:rep-xm}
    \begin{adjustbox}{width=\linewidth}
        \begin{tabular}{l|c|c|c|c|}
            \toprule
            Method & \multicolumn{2}{c|}{\xmdouble} & \multicolumn{2}{c|}{\textsc{Filter + \xmdouble}} \\
            \midrule
            Metric & Min-eig & Suboptimality-gap &  Min-eig & Suboptimality-gap\\
            \midrule\midrule
            Room0-100 & $1.1\times10^{-5}$ & $1.9\times10^{-6}$ & $1.1\times10^{-5}$& $7.4\times10^{-8}$\\
            \midrule
            Room0-2000 & $3.5\times10^{-8}$ &$2.2\times10^{-5}$ & $-6.3\times10^{-8}$ &$5.9\times10^{-6}$\\
            \midrule
            Room1-100 & $2.6\times10^{-6}$ & $7.9\times10^{-4}$ & $1.3\times10^{-5}$ &$1.8\times10^{-6}$\\
            \midrule
            Room1-2000 & $5.8\times10^{-8}$ & $2.2\times10^{-4}$ & $-4.6\times10^{-7}$ &$4.5\times10^{-6}$\\
            \midrule
            Office0-100 & $-8.0\times10^{-7}$ & $1.4\times10^{-7}$ & $-1.0\times10^{-6}$ &$1.1\times10^{-7}$\\
            \midrule
            Office0-2000 & $4.7\times10^{-7}$ & $5.2\times10^{-4}$ & $-1.5\times10^{-6}$ &$6.7\times10^{-6}$\\
            \midrule
            Office1-100 & $-8.0\times10^{-7}$ & $1.0\times10^{-7}$ & $8.5\times10^{-7}$ &$2.4\times10^{-8}$\\
            \midrule
            Office1-2000 & $-4.2\times10^{-8}$ & $2.3\times10^{-5}$ & $6.8\times10^{-8}$ &$3.8\times10^{-4}$\\
            \midrule
            \bottomrule
        \end{tabular}        
\end{adjustbox}
\end{table}

We then test on the Replica dataset~\cite{zhu2022nice, sucar2021imap, replica19arxiv}, which contains synthetic images of different virtual scenes. 

\textbf{Setup, baselines, metrics}.
We use the groundtruth depth map but employ the full \xmsfm pipeline. 
We compare \xmsfm, both with and without two-view filtering, against \colmap and \glomap. The evaluation metrics remain the same. For runtime analysis, we categorize all components preceding our \nameshort solver---including Matching, Indexing, Depth Estimation, Filtering, and Matrix Construction---as preprocessing time. The indexing, filtering and matrix construction components can be further accelerated in CUDA. Similarly, for \glomap and \colmap, all steps prior to global positioning and bundle adjustment are counted as preprocessing time.

\textbf{Results}.
Results are presented in \prettyref{tab:replica} and \prettyref{tab:rep-xm}. Each Replica dataset contains 2000 frames, but for a diverse comparison across different dataset sizes, we sample the first 100 frames from each dataset as a separate experiment. ``Room0-100'' refers to the first 100 frames, while ``Room0-2000'' represents the full dataset.  

\emph{\nameshort consistently outperforms baselines by 100 to 1000 times in solver speed}, solving almost all 2000-frame datasets within 10 seconds. At the same time, \nameshort maintains high accuracy, achieving a median translation error of just 1\%. In practice, a 0.1\% and 1\% translation error yield nearly identical reconstruction quality, as illustrated in \prettyref{fig:visualization-replica}.

Additionally, we provide a runtime breakdown for \nameshort in \prettyref{app:breakdown-time}. The solver time is negligible, appearing as only a thin bar in the chart. Matching, indexing, and filtering are the most time-consuming components, with the latter two planned for CUDA implementation as future work.

\begin{tcolorbox}
    \begin{center}
        \vspace{-1mm}
        \textbf{Takeaway}
        \vspace{-1mm}
    \end{center} 
    $\bullet$ With ground truth depth and \colmap matchings, minor filter refinement achieves results comparable to \colmap and \glomap.  

    $\bullet$ \nameshort remains highly efficient and scalable.  
\end{tcolorbox}

\subsection{Mip-Nerf and Zip-Nerf Dataset}\label{sec:exp-mipnerf}

\begin{table}[t]
    \centering
    \caption{\textnormal{Results on the Mip-Nerf and Zip-Nerf datasets. \colmap is very slow, while both \nameshort and \glomap achieve comparable accuracy. However, \nameshort is significantly faster.}}  
    \label{tab:mipnerf}
    \begin{adjustbox}{width=\linewidth}
        \begin{tabular}{l|l|l|l|l|l|l|}
            \toprule
            \multirow{3}{*}{Method} & \multirow{3}{*}{Metrics} & \multicolumn{5}{c|}{Datasets} \\
            \cmidrule{3-7}
            & & \stackon{-187865}{Kitchen-279} & \stackon{-106858}{Garden-185} & \stackon{-41866}{Bicycle-194} & \stackon{-111783}{Room-311} & \stackon{-416665}{Alameda-1724} \\
            \midrule
            \multirow{6}{*}{\stackon{\textsc{ + Ceres}}{\textsc{Filter + \xmdouble}}}
            & \stackon{(\nameshort + \ceres)}{Solver Time} & \stackon{$\bf 19.19$}{$\bf 0.8~+ $} & \stackon{$\bf 3.98 $}{$\bf 0.65~+ $} & \stackon{$\bf 22.47$}{$\bf 0.58~+ $} & \stackon{$\bf 24.35$}{$\bf 1.01~+ $} & \stackon{$\bf 218.13$}{$\bf 62.0~+ $} \\
            \cmidrule{2-7}
            & Processing Time & $229.82$ & $174.57$ & $145.75$ & $223.44$ & $3348.46$ \\
            \cmidrule{2-7}
            & ATE-T & $0.008$ & $0.002$ & $0.019$ & $0.002$ & $0.009$ \\
            & ATE-R & $0.072^\circ$ & $0.021^\circ$ & $0.229^\circ$ & $0.06^\circ$ & $0.317^\circ$ \\
            \cmidrule{2-7}
            & RPE-T & $0.005$ & $0.003$ & $0.028$ & $0.002$ & $0.014$ \\
            & RPE-R & $0.054^\circ$ & $0.025^\circ$ & $0.219^\circ$ & $0.084^\circ$ & $0.493^\circ$ \\
            \midrule
            \multirow{6}{*}{\textsc{glomap}}
            & Solver Time & $363.22$ & $193.86$ & $77.76$ & $269.39$ & $1169.83$ \\
            \cmidrule{2-7}
            & Processing Time & $124.23$ & $90.7$ & $79.84$ & $96.26$ & $2631.69$ \\
            \cmidrule{2-7}
            & ATE-T & $0.016$ & $0.013$ & $0.039$ & $0.003$ & $0.001$ \\
            & ATE-R & $0.107^\circ$ & $0.061^\circ$ & $0.51^\circ$ & $0.083^\circ$ & $0.074^\circ$ \\
            \cmidrule{2-7}
            & RPE-T & $0.023$ & $0.02$ & $0.056$ & $0.003$ & $0.002$ \\
            & RPE-R & $0.135^\circ$ & $0.067^\circ$ & $0.114^\circ$ & $0.029^\circ$ & $0.053^\circ$ \\
            \midrule
            \multirow{6}{*}{\textsc{colmap}}
            & Solver Time & $1422.96$ & $486.84$ & $273.84$ & $958.5$ & $79278.36$\\
            \cmidrule{2-7}
            & Processing Time & $\bf 97.26$ & $\bf 78.1$ & $\bf 75.16$ & $\bf 62.47$ & $\bf 2494.29$ \\
            \cmidrule{2-7}
            & ATE-T & $0.0$ & $0.0$ & $0.0$ & $0.0$ & $0.0$ \\
            & ATE-R & $0.0^\circ$ & $0.0^\circ$ & $0.0^\circ$ & $0.0^\circ$ & $0.0^\circ$ \\
            \cmidrule{2-7}
            & RPE-T & $0.0$ & $0.0$ & $0.0$ & $0.0$ & $0.0$ \\
            & RPE-R & $0.0^\circ$ & $0.0^\circ$ & $0.0^\circ$ & $0.0^\circ$ & $0.0^\circ$ \\
            \bottomrule
        \end{tabular}
    \end{adjustbox}
\end{table}

\begin{table}[t]
    \centering
    \caption{\textnormal{Results on the Mip-Nerf and Zip-Nerf datasets (min-eig and suboptimality). All datasets are solved to global minimum.
    }}
    \label{tab:mipnerf-xm}
    \begin{adjustbox}{width=\linewidth}
        \begin{tabular}{l|c|c|c|c|c|}
            \toprule
            Metric & Kitchen & Garden & Bicycle & Room & Alameda \\
            \midrule
            Min-eig & $1.1\times10^{-10}$ & $-3.6\times10^{-9}$ & $4.8\times10^{-8}$ & $-4.6\times10^{-7}$ & $-3.8\times10^{-7}$ \\
            \midrule
            Suboptimality-gap & $8.9\times10^{-8}$ & $8.8\times10^{-9}$ & $2.5\times10^{-5}$ & $4.5\times10^{-6}$ & $8.1\times10^{-2}$ \\
            \bottomrule
        \end{tabular}        
    \end{adjustbox}
\end{table}

\begin{figure}
    \centering
    \includegraphics[width=0.9\linewidth]{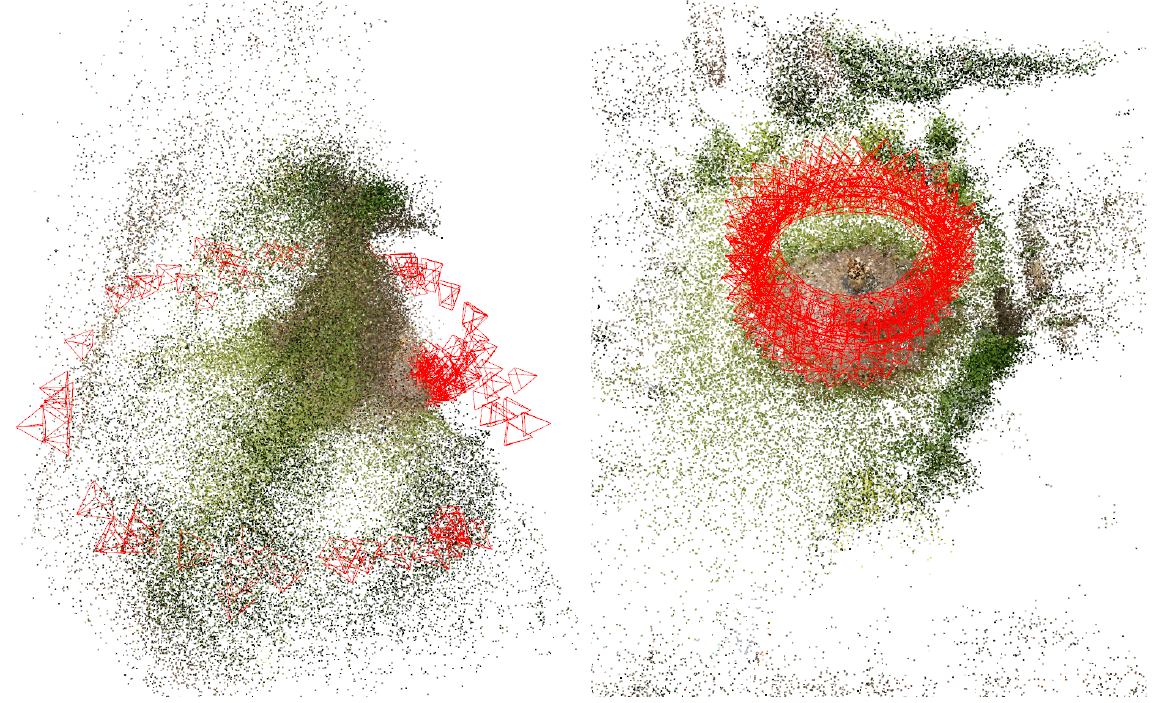}
    \caption{Illustration of local minimum on the Mip-Nerf dataset. \textbf{Left:} Solution of \prettyref{eq:problem-bm} with $r = 3$, where the solver gets stuck in a poor local minimum. \textbf{Right:} Solution after increasing the rank to 4, which successfully escapes the local minimum.
    }
    \label{fig:mipnerf-rank4}
    \vspace{-3mm}
\end{figure}

Mip-Nerf and Zip-Nerf \cite{barron2022mipnerf360,barron2023zipnerf} are real-world image datasets around a single object. We evaluate learned depth and downstream novel view synthesis tasks on these datasets.

\textbf{Setup, baselines, metrics}.
As these datasets are generated by \colmap, we use its camera poses as ground truth to benchmark \nameshort against \colmap and \glomap. To address inaccuracies in learned depth, we apply \ceres refinement, incorporating its runtime into the solver time, denoted as ``\nameshort + \ceres.'' All other evaluation metrics remain unchanged.

\textbf{Results}.
Results are presented in \prettyref{tab:mipnerf}, with suboptimality detailed in \prettyref{tab:mipnerf-xm}. \emph{While adding \ceres increases solver time, the overall runtime remains 10 to 100 times faster than the baselines.} On the garden and room datasets, we achieve a 0.2\% error, demonstrating the same accuracy as the baselines, while on others, the error may be slightly higher. However, as shown in \prettyref{fig:visualization-mipnerf}, this has no noticeable impact on downstream 3D Gaussian Splatting tasks~\cite{kerbl3Dgaussians}.

A runtime breakdown for \nameshort is provided in \prettyref{app:breakdown-time}. Matching and indexing remain slow, while depth estimation now takes even longer. This is due to (a) foundation models requiring much time for estimation and (b) depth estimation runtime scales linearly with the number of frames, whereas Mip-Nerf datasets are relatively small.  

\textbf{The need for convex relaxation}. 
In \prettyref{fig:mipnerf-rank4}, we demonstrate that directly solving \prettyref{eq:problem-bm} with rank 3 (equivalently \prettyref{eq:problem-qcqp}) can lead to a local minimum. Specifically, the solver terminates with a minimal eigenvalue of $Z(y)$ at $-6.2 \times 10^{2}$, resulting in a messy reconstruction. However, by increasing the rank to 4, the solver escapes the local minimum and achieves the global minimum, indicating that while the relaxation remains tight, the Burer-Monteiro method requires a higher rank to find the global minimum. 

\begin{tcolorbox}
    \begin{center}
        \vspace{-1mm}
        \textbf{Takeaway}
        \vspace{-1mm}
    \end{center} 
    $\bullet$ With \ceres refinement to mitigate depth errors, \nameshort achieves nearly the same accuracy as \colmap and \glomap. Moreover, this has no impact on downstream novel view synthesis tasks.  

    $\bullet$ \nameshort maintains a significant speed advantage, even with learned depth and real-world data.  

    $\bullet$ BM factorization and the Riemannian staircase effectively escape local minimum.  
\end{tcolorbox}

\subsection{IMC, TUM and C3VD Datasets}\label{sec:exp-imc-tum}

\begin{table}[t]
    \centering
    \caption{\textnormal{Results on the IMC datasets. \xmsfm achieves comparable accuracy while being much more scalable.
    }}
    \label{tab:imc}
    \begin{adjustbox}{width=\linewidth}
        \begin{tabular}{l|c|c|c|c|c|c|}
            \toprule
            Datasets & \multicolumn{3}{c|}{\textsc{Filter + \xmdouble + Ceres}}& \multicolumn{3}{c|}{\glomap}   \\
            \midrule
            Metrics & \stackon{Solver\,Time}{Processing\,Time} & \stackon{ATE-T}{ATE-R} & \stackon{RPE-T}{RPE-R} & \stackon{Solver\,Time}{Processing\,Time} & \stackon{ATE-T}{ATE-R} & \stackon{RPE-T}{RPE-R}  \\
            \midrule\midrule
            \stackon{-203244}{Rome-2063}
            & \stackon{$\bf 19.17 + 300.16 $}{$ 4877.89 $} & \stackon{$ 0.021 $}{$ 1.252 ^{\circ}$} & \stackon{$ 0.035 $}{$ 0.72 ^{\circ}$}
            & \stackon{$ 9813.04 $}{$\bf3459.04$} & \stackon{$\bf 0.003 $}{$\bf 0.09 ^{\circ}$} & \stackon{$\bf 0.005 $}{$\bf 0.069 ^{\circ}$}
            \\
            \midrule
            \stackon{-62561}{Gate-1363}
            & \stackon{$\bf 2.92 + 99.66 $}{$ 1349.74 $} & \stackon{$ 0.038 $}{$ 1.329 ^{\circ}$} & \stackon{$ 0.089 $}{$ 2.022 ^{\circ}$}
            & \stackon{$ 2263.79 $}{$\bf712.56$} & \stackon{$\bf 0.018 $}{$\bf 0.41 ^{\circ}$} & \stackon{$\bf 0.047 $}{$\bf 0.19 ^{\circ}$}
            \\
            \midrule
            \stackon{-78026}{Temple-904}
            & \stackon{$\bf 1.34 + 93.5 $}{$ 1026.65 $} & \stackon{$\bf 0.008 $}{$ 0.545 ^{\circ}$} & \stackon{$\bf 0.013 $}{$ 0.254 ^{\circ}$}
            & \stackon{$ 1560.86 $}{$\bf600.4$} & \stackon{$ 0.012 $}{$\bf 0.276 ^{\circ}$} & \stackon{$ 0.03 $}{$\bf 0.216 ^{\circ}$}
            \\
            \midrule
            \stackon{-314422}{Paris-3765}
            & \stackon{$\bf 22.17 + 304.49 $}{$ 16645.16 $} & \stackon{$\bf 0.008 $}{$ 0.473 ^{\circ}$} & \stackon{$\bf 0.017 $}{$ 0.561 ^{\circ}$}
            & \stackon{$ 74560.1 $}{$\bf10598.8$} & \stackon{$ 0.009 $}{$\bf 0.117 ^{\circ}$} & \stackon{$\bf 0.017 $}{$\bf 0.087 ^{\circ}$}
            \\
            \bottomrule
        \end{tabular}
    \end{adjustbox}
    \vspace{-4mm}
\end{table}

\begin{table}[t]
    \centering
    \caption{\textnormal{Results on the TUM datasets. The error is slightly higher due to low-resolution images.}}
    \label{tab:tum}
    \begin{adjustbox}{width=\linewidth}
        \begin{tabular}{l|c|c|c|}
            \toprule
            Metrics & Solver Time / Processing Time & ATE-T / ATE-R & RPE-T / RPE-R \\
            \midrule\midrule
             fr1/xyz-798-26078 & $0.63 + 9.69 \, / \, 498.49$ & $0.035 \, / \, 3.148^{\circ}$ & $0.05 \, / \, 1.134^{\circ}$ \\
            \midrule
             fr1/rpy-723-26071 & $0.88 + 18.8 \, / \, 329.84$ & $0.023 \, / \, 6.784^{\circ}$ & $0.038 \, / \, 2.495^{\circ}$ \\
            \midrule
             fr1/desk-613-38765 & $0.63 + 7.95 \, / \, 201.17$ & $0.024 \, / \, 2.831^{\circ}$ & $0.041 \, / \, 2.277^{\circ}$ \\
            \midrule
             fr1/room-1362-86634 & $10.17 + 182.29 \, / \, 464.4$ & $0.049 \, / \, 3.718^{\circ}$ & $0.086 \, / \, 2.196^{\circ}$ \\
            \bottomrule
        \end{tabular}
    \end{adjustbox}
\end{table}

\begin{table}[t]
    \centering
    \caption{\textnormal{Results on the C3VD datasets. Ground truth depth is significantly more accurate than learned depth, while they both fail on the last dataset because of dark environment.}}
    \label{tab:c3vd}
    \begin{adjustbox}{width=\linewidth}
        \begin{tabular}{l|c|c|c|c|c|c|}
            \toprule
            Datasets & \multicolumn{3}{c|}{\textsc{Filter + \xmdouble + GT-depth}} & \multicolumn{3}{c|}{\textsc{Filter + \xmdouble}}\\
            \midrule
            Metrics & \stackon{Solver\,Time}{Processing\,Time} & \stackon{ATE-T}{ATE-R} & \stackon{RPE-T}{RPE-R}& \stackon{Solver\,Time}{Processing\,Time} & \stackon{ATE-T}{ATE-R} & \stackon{RPE-T}{RPE-R}\\
            \midrule\midrule
            medical5
            & \stackon{$ 0.31 $}{$ 55.45 $} & \stackon{$ 0.009 $}{$ 3.284 ^{\circ}$} & \stackon{$ 0.015 $}{$ 0.647 ^{\circ}$}
            & \stackon{$ 0.34 $}{$ 30.19 $} & \stackon{$ 0.109 $}{$ 11.64 ^{\circ}$} & \stackon{$ 0.169 $}{$ 3.191 ^{\circ}$}
            \\
            \midrule
            medical6
            & \stackon{$ 0.46 $}{$ 171.35 $} & \stackon{$ 0.014 $}{$ 4.102 ^{\circ}$} & \stackon{$ 0.019 $}{$ 0.851 ^{\circ}$}
            & \stackon{$ 0.47 $}{$ 166.19 $} & \stackon{$ 0.262 $}{$ 116.79 ^{\circ}$} & \stackon{$ 0.352 $}{$ 4.729 ^{\circ}$}
            \\
            \midrule
            medical7
            & \stackon{$ 0.81 $}{$ 261.58 $} & \stackon{$ 0.01 $}{$ 3.69 ^{\circ}$} & \stackon{$ 0.013 $}{$ 1.231 ^{\circ}$}
            & \stackon{$ 0.89 $}{$ 250.48 $} & \stackon{$ 0.047 $}{$ 131.686 ^{\circ}$} & \stackon{$ 0.062 $}{$ 3.564 ^{\circ}$}
            \\
            \midrule
            medical8
            & \stackon{$ 0.23 $}{$ 56.84 $} & \stackon{$ 0.72 $}{$ 66.107 ^{\circ}$} & \stackon{$ 1.196 $}{$ 0.062 ^{\circ}$}
            & \stackon{$ 0.22 $}{$ 38.25 $} & \stackon{$ 0.452 $}{$ 159.909 ^{\circ}$} & \stackon{$ 0.704 $}{$ 0.091 ^{\circ}$}
            \\
            \bottomrule
        \end{tabular}
    \end{adjustbox}
\end{table}

Followed by Mip-Nerf, we step to the IMC PhotoTourism dataset \cite{imc2023}, TUM dataset \cite{sturm12tum}, and C3VD medical dataset \cite{bobrow2023}. These datasets are quite challenging because of varying environments and low-quality images.

\textbf{Setup, baselines, metrics}.
We only compare our \nameshort solver against \glomap on IMC datasets. In these three datasets we add both outlier filtering and \xmdouble. In C3VD we use the ground truth depth map and the learned depth from a medical-specific depth prediction model~\cite{paruchuri2024leveraging}. 

\textbf{Results}.
The results are presented in \prettyref{tab:imc}, \prettyref{tab:tum}, and \prettyref{tab:c3vd}. The IMC datasets consist of large image collections capturing some famous landmarks. As a result, \glomap requires an extremely long runtime, exceeding 20 hours for the largest dataset. Our accuracy is comparable to \glomap, with visualizations provided in \prettyref{fig:visualization-imc}. The TUM and C3VD datasets produce high-quality reconstructions, though accuracy is affected by the complexity of the environment and insufficient lighting. Visualizations of these reconstructions are provided in \prettyref{fig:visualization-tum} and \prettyref{fig:visualization-c3vd}.  

\begin{tcolorbox}
    \begin{center}
        \vspace{-1mm}
        \textbf{Takeaway}
        \vspace{-1mm}
    \end{center} 
    \nameshort remains efficient and scalable across diverse datasets, achieving results comparable to \glomap while being significantly faster.
\end{tcolorbox}


\section{Conclusion}
\label{sec:conclusion}
We proposed \nameshort, a scalable and initialization-free solver for global bundle adjustment, leveraging learned depth and convex optimization. By relaxing scaled bundle adjustment as a convex SDP and solving it efficiently with Burer-Monteiro factorization and a CUDA-based trust-region Riemannian optimizer, \nameshort achieved certifiable global optimality at extreme scales. Integrated into the \xmsfm pipeline, it maintains the accuracy of existing SfM methods while being significantly faster and more scalable.  

\textbf{Limitation and future work}. First, while our \nameshort solver outperforms baselines in speed, it can be sensitive to noise and outliers. Future work includes refining the filtering process and developing better methods to handle outliers. Second, our GPU solver is built on the \texttt{cuBLAS} dense matrix-vector multiplication library, whereas SLAM camera sequences often have sparse patterns. Extending the \nameshort solver to support sparse matrix-vector multiplications would enhance its applicability to SLAM. Third, our \xmsfm pipeline still has components that potentially can be accelerated 100 to 1000 times through CUDA implementation, e.g., filtering and matrix construction. Lastly, we acknowledge that \nameshort relies on a well-trained depth prediction model. While generic models are available, training effective depth predictors can be particularly challenging in data-scarce domains such as medical or space applications.


\section*{Acknowledgements}

This work was partially funded by Office of Naval Research grant N00014-25-1-2322. We thank Shucheng Kang for the help on CUDA programming; Xihang Yu and Luca Carlone for discussions about depth prediction models; and members of the Harvard Computational Robotics Group for various explorations throughout the project.
\bibliographystyle{plainnat}
\bibliography{refs}

\begin{thebibliography}{53}
\providecommand{\natexlab}[1]{#1}
\providecommand{\url}[1]{\texttt{#1}}
\expandafter\ifx\csname urlstyle\endcsname\relax
  \providecommand{\doi}[1]{doi: #1}\else
  \providecommand{\doi}{doi: \begingroup \urlstyle{rm}\Url}\fi

\bibitem[Absil et~al.(2008)Absil, Mahony, and Sepulchre]{absil2008optimization}
P-A Absil, Robert Mahony, and Rodolphe Sepulchre.
\newblock \emph{Optimization algorithms on matrix manifolds}.
\newblock Princeton University Press, 2008.

\bibitem[Agarwal et~al.(2010)Agarwal, Snavely, Seitz, and Szeliski]{agarwal2010bundle}
Sameer Agarwal, Noah Snavely, Steven~M Seitz, and Richard Szeliski.
\newblock Bundle adjustment in the large.
\newblock In \emph{Computer Vision--ECCV 2010: 11th European Conference on Computer Vision, Heraklion, Crete, Greece, September 5-11, 2010, Proceedings, Part II 11}, pages 29--42. Springer, 2010.

\bibitem[Agarwal et~al.(2012)Agarwal, Mierle, et~al.]{agarwal2012ceres}
Sameer Agarwal, Keir Mierle, et~al.
\newblock Ceres solver: Tutorial \& reference.
\newblock \emph{Google Inc}, 2\penalty0 (72):\penalty0 8, 2012.

\bibitem[Antonante et~al.(2021)Antonante, Tzoumas, Yang, and Carlone]{antonante2021outlier}
Pasquale Antonante, Vasileios Tzoumas, Heng Yang, and Luca Carlone.
\newblock Outlier-robust estimation: Hardness, minimally tuned algorithms, and applications.
\newblock \emph{IEEE Transactions on Robotics}, 38\penalty0 (1):\penalty0 281--301, 2021.

\bibitem[ApS(2019)]{aps2019mosek}
Mosek ApS.
\newblock Mosek optimization toolbox for matlab.
\newblock \emph{User’s Guide and Reference Manual, Version}, 4\penalty0 (1), 2019.

\bibitem[Barfoot et~al.(2023)Barfoot, Holmes, and D{\"u}mbgen]{barfoot2023certifiably}
Timothy~D Barfoot, Connor Holmes, and Frederike D{\"u}mbgen.
\newblock Certifiably optimal rotation and pose estimation based on the cayley map.
\newblock \emph{The International Journal of Robotics Research}, page 02783649241269337, 2023.

\bibitem[Barron et~al.(2022)Barron, Mildenhall, Verbin, Srinivasan, and Hedman]{barron2022mipnerf360}
Jonathan~T. Barron, Ben Mildenhall, Dor Verbin, Pratul~P. Srinivasan, and Peter Hedman.
\newblock Mip-nerf 360: Unbounded anti-aliased neural radiance fields.
\newblock \emph{CVPR}, 2022.

\bibitem[Barron et~al.(2023)Barron, Mildenhall, Verbin, Srinivasan, and Hedman]{barron2023zipnerf}
Jonathan~T. Barron, Ben Mildenhall, Dor Verbin, Pratul~P. Srinivasan, and Peter Hedman.
\newblock Zip-nerf: Anti-aliased grid-based neural radiance fields.
\newblock \emph{ICCV}, 2023.

\bibitem[Bobrow et~al.(2023)Bobrow, Golhar, Vijayan, Akshintala, Garcia, and Durr]{bobrow2023}
Taylor~L Bobrow, Mayank Golhar, Rohan Vijayan, Venkata~S Akshintala, Juan~R Garcia, and Nicholas~J Durr.
\newblock Colonoscopy 3d video dataset with paired depth from 2d-3d registration.
\newblock \emph{Medical Image Analysis}, page 102956, 2023.

\bibitem[Bochkovskii et~al.(2024)Bochkovskii, Delaunoy, Germain, Santos, Zhou, Richter, and Koltun]{bochkovskii2024depth}
Aleksei Bochkovskii, Ama{\"e}l Delaunoy, Hugo Germain, Marcel Santos, Yichao Zhou, Stephan~R Richter, and Vladlen Koltun.
\newblock Depth pro: Sharp monocular metric depth in less than a second.
\newblock \emph{arXiv preprint arXiv:2410.02073}, 2024.

\bibitem[Boumal(2023)]{boumal2023introduction}
Nicolas Boumal.
\newblock \emph{An introduction to optimization on smooth manifolds}.
\newblock Cambridge University Press, 2023.

\bibitem[Boumal et~al.(2014)Boumal, Mishra, Absil, and Sepulchre]{boumal2014manopt}
Nicolas Boumal, Bamdev Mishra, P-A Absil, and Rodolphe Sepulchre.
\newblock Manopt, a matlab toolbox for optimization on manifolds.
\newblock \emph{The Journal of Machine Learning Research}, 15\penalty0 (1):\penalty0 1455--1459, 2014.

\bibitem[Burer and Monteiro(2003)]{burer2003nonlinear}
Samuel Burer and Renato~DC Monteiro.
\newblock A nonlinear programming algorithm for solving semidefinite programs via low-rank factorization.
\newblock \emph{Mathematical programming}, 95\penalty0 (2):\penalty0 329--357, 2003.

\bibitem[Carlone et~al.(2015)Carlone, Rosen, Calafiore, Leonard, and Dellaert]{carlone2015lagrangian}
Luca Carlone, David~M Rosen, Giuseppe Calafiore, John~J Leonard, and Frank Dellaert.
\newblock Lagrangian duality in 3d slam: Verification techniques and optimal solutions.
\newblock In \emph{2015 IEEE/RSJ International Conference on Intelligent Robots and Systems (IROS)}, pages 125--132. IEEE, 2015.

\bibitem[Chaudhury et~al.(2015)Chaudhury, Khoo, and Singer]{chaudhury2015global}
Kunal~N Chaudhury, Yuehaw Khoo, and Amit Singer.
\newblock Global registration of multiple point clouds using semidefinite programming.
\newblock \emph{SIAM Journal on Optimization}, 25\penalty0 (1):\penalty0 468--501, 2015.

\bibitem[Chow et~al.(2023)Chow, Trulls, HCL-Jevster, Yi, lcmrll, old ufo, Dane, tanjigou, WastedCode, and Sun]{imc2023}
Ashley Chow, Eduard Trulls, HCL-Jevster, Kwang~Moo Yi, lcmrll, old ufo, Sohier Dane, tanjigou, WastedCode, and Weiwei Sun.
\newblock Image matching challenge 2023.
\newblock \url{https://kaggle.com/competitions/image-matching-challenge-2023}, 2023.
\newblock Kaggle.

\bibitem[Dellaert and Contributors(2022)]{dellaert2022gtsam}
Frank Dellaert and GTSAM Contributors.
\newblock borglab/gtsam, 2022.
\newblock URL \url{https://github.com/borglab/gtsam}.

\bibitem[Dellaert et~al.(2020)Dellaert, Rosen, Wu, Mahony, and Carlone]{dellaert2020shonan}
Frank Dellaert, David~M Rosen, Jing Wu, Robert Mahony, and Luca Carlone.
\newblock Shonan rotation averaging: Global optimality by surfing so (p)\^{} n so (p) n.
\newblock In \emph{Computer Vision--ECCV 2020: 16th European Conference, Glasgow, UK, August 23--28, 2020, Proceedings, Part VI 16}, pages 292--308. Springer, 2020.

\bibitem[Fan et~al.(2023)Fan, Ortiz, Hsiao, Monge, Dong, Murphey, and Mukadam]{fan2023decentralization}
Taosha Fan, Joseph Ortiz, Ming Hsiao, Maurizio Monge, Jing Dong, Todd Murphey, and Mustafa Mukadam.
\newblock Decentralization and acceleration enables large-scale bundle adjustment.
\newblock \emph{arXiv preprint arXiv:2305.07026}, 2023.

\bibitem[Fischler and Bolles(1981)]{fischler1981random}
Martin~A Fischler and Robert~C Bolles.
\newblock Random sample consensus: a paradigm for model fitting with applications to image analysis and automated cartography.
\newblock \emph{Communications of the ACM}, 24\penalty0 (6):\penalty0 381--395, 1981.

\bibitem[Garcia-Salguero and Gonzalez-Jimenez(2024)]{garcia2024certifiable}
Mercedes Garcia-Salguero and Javier Gonzalez-Jimenez.
\newblock Certifiable planar relative pose estimation with gravity prior.
\newblock \emph{Computer Vision and Image Understanding}, 239:\penalty0 103887, 2024.

\bibitem[Garcia-Salguero et~al.(2021)Garcia-Salguero, Briales, and Gonzalez-Jimenez]{garcia2021certifiable}
Mercedes Garcia-Salguero, Jesus Briales, and Javier Gonzalez-Jimenez.
\newblock Certifiable relative pose estimation.
\newblock \emph{Image and Vision Computing}, 109:\penalty0 104142, 2021.

\bibitem[Holmes and Barfoot(2023)]{holmes2023efficient}
Connor Holmes and Timothy~D Barfoot.
\newblock An efficient global optimality certificate for landmark-based slam.
\newblock \emph{IEEE Robotics and Automation Letters}, 8\penalty0 (3):\penalty0 1539--1546, 2023.

\bibitem[Holmes et~al.(2024)Holmes, D{\"u}mbgen, and Barfoot]{holmes2024sdprlayers}
Connor Holmes, Frederike D{\"u}mbgen, and Timothy~D Barfoot.
\newblock Sdprlayers: Certifiable backpropagation through polynomial optimization problems in robotics.
\newblock \emph{arXiv preprint arXiv:2405.19309}, 2024.

\bibitem[Horn(1987)]{horn1987closed}
Berthold~KP Horn.
\newblock Closed-form solution of absolute orientation using unit quaternions.
\newblock \emph{Josa a}, 4\penalty0 (4):\penalty0 629--642, 1987.

\bibitem[Iglesias et~al.(2020)Iglesias, Olsson, and Kahl]{iglesias2020global}
Jos{\'e}~Pedro Iglesias, Carl Olsson, and Fredrik Kahl.
\newblock Global optimality for point set registration using semidefinite programming.
\newblock In \emph{Proceedings of the IEEE/CVF conference on computer vision and pattern recognition}, pages 8287--8295, 2020.

\bibitem[Journ{\'e}e et~al.(2010)Journ{\'e}e, Bach, Absil, and Sepulchre]{journee2010low}
Michel Journ{\'e}e, Francis Bach, P-A Absil, and Rodolphe Sepulchre.
\newblock Low-rank optimization on the cone of positive semidefinite matrices.
\newblock \emph{SIAM Journal on Optimization}, 2010.

\bibitem[Kang et~al.(2024)Kang, Xu, Sarva, Liang, and Yang]{kang2024fast}
Shucheng Kang, Xiaoyang Xu, Jay Sarva, Ling Liang, and Heng Yang.
\newblock Fast and certifiable trajectory optimization.
\newblock \emph{arXiv preprint arXiv:2406.05846}, 2024.

\bibitem[Kerbl et~al.(2023)Kerbl, Kopanas, Leimk{\"u}hler, and Drettakis]{kerbl3Dgaussians}
Bernhard Kerbl, Georgios Kopanas, Thomas Leimk{\"u}hler, and George Drettakis.
\newblock 3d gaussian splatting for real-time radiance field rendering.
\newblock \emph{ACM Transactions on Graphics}, 42\penalty0 (4), July 2023.
\newblock URL \url{https://repo-sam.inria.fr/fungraph/3d-gaussian-splatting/}.

\bibitem[Lowe(1999)]{lowe1999sift}
David~G Lowe.
\newblock Object recognition from local scale-invariant features.
\newblock In \emph{Proceedings of the seventh IEEE international conference on computer vision}, volume~2, pages 1150--1157. Ieee, 1999.

\bibitem[Olsson et~al.(2010)Olsson, Eriksson, and Hartley]{olsson2010outlier}
Carl Olsson, Anders Eriksson, and Richard Hartley.
\newblock Outlier removal using duality.
\newblock In \emph{2010 IEEE Computer Society Conference on Computer Vision and Pattern Recognition}, pages 1450--1457. IEEE, 2010.

\bibitem[Pan et~al.(2025)Pan, Bar{\'a}th, Pollefeys, and Sch{\"o}nberger]{pan2025global}
Linfei Pan, D{\'a}niel Bar{\'a}th, Marc Pollefeys, and Johannes~L Sch{\"o}nberger.
\newblock Global structure-from-motion revisited.
\newblock In \emph{European Conference on Computer Vision}, pages 58--77. Springer, 2025.

\bibitem[Papalia et~al.(2024)Papalia, Fishberg, O'Neill, How, Rosen, and Leonard]{papalia2024certifiably}
Alan Papalia, Andrew Fishberg, Brendan~W O'Neill, Jonathan~P How, David~M Rosen, and John~J Leonard.
\newblock Certifiably correct range-aided slam.
\newblock \emph{IEEE Transactions on Robotics}, 2024.

\bibitem[Paruchuri et~al.(2024)Paruchuri, Ehrenstein, Wang, Fried, Pizer, Niethammer, and Sengupta]{paruchuri2024leveraging}
Akshay Paruchuri, Samuel Ehrenstein, Shuxian Wang, Inbar Fried, Stephen~M Pizer, Marc Niethammer, and Roni Sengupta.
\newblock Leveraging near-field lighting for monocular depth estimation from endoscopy videos.
\newblock \emph{arXiv preprint arXiv:2403.17915}, 2024.

\bibitem[Piccinelli et~al.(2024)Piccinelli, Yang, Sakaridis, Segu, Li, Van~Gool, and Yu]{piccinelli2024unidepth}
Luigi Piccinelli, Yung-Hsu Yang, Christos Sakaridis, Mattia Segu, Siyuan Li, Luc Van~Gool, and Fisher Yu.
\newblock Unidepth: Universal monocular metric depth estimation.
\newblock In \emph{Proceedings of the IEEE/CVF Conference on Computer Vision and Pattern Recognition}, pages 10106--10116, 2024.

\bibitem[Ren et~al.(2022)Ren, Liang, Yan, Mai, Liu, and Liu]{ren2022megba}
Jie Ren, Wenteng Liang, Ran Yan, Luo Mai, Shiwen Liu, and Xiao Liu.
\newblock Megba: A gpu-based distributed library for large-scale bundle adjustment.
\newblock In \emph{European Conference on Computer Vision}, pages 715--731. Springer, 2022.

\bibitem[Rosen(2021)]{rosen2021scalable}
David~M Rosen.
\newblock Scalable low-rank semidefinite programming for certifiably correct machine perception.
\newblock In \emph{Algorithmic Foundations of Robotics XIV: Proceedings of the Fourteenth Workshop on the Algorithmic Foundations of Robotics 14}, pages 551--566. Springer, 2021.

\bibitem[Rosen et~al.(2019)Rosen, Carlone, Bandeira, and Leonard]{rosen2019se}
David~M Rosen, Luca Carlone, Afonso~S Bandeira, and John~J Leonard.
\newblock Se-sync: A certifiably correct algorithm for synchronization over the special euclidean group.
\newblock \emph{The International Journal of Robotics Research}, 38\penalty0 (2-3):\penalty0 95--125, 2019.

\bibitem[Sch\"{o}nberger and Frahm(2016)]{schoenberger2016sfm}
Johannes~Lutz Sch\"{o}nberger and Jan-Michael Frahm.
\newblock Structure-from-motion revisited.
\newblock In \emph{Conference on Computer Vision and Pattern Recognition (CVPR)}, 2016.

\bibitem[Sim and Hartley(2006)]{sim2006removing}
Kristy Sim and Richard Hartley.
\newblock Removing outliers using the $\ell_{\infty}$ norm.
\newblock In \emph{2006 IEEE Computer Society Conference on Computer Vision and Pattern Recognition (CVPR'06)}, volume~1, pages 485--494. IEEE, 2006.

\bibitem[Straub et~al.(2019)Straub, Whelan, Ma, Chen, Wijmans, Green, Engel, Mur-Artal, Ren, Verma, Clarkson, Yan, Budge, Yan, Pan, Yon, Zou, Leon, Carter, Briales, Gillingham, Mueggler, Pesqueira, Savva, Batra, Strasdat, Nardi, Goesele, Lovegrove, and Newcombe]{replica19arxiv}
Julian Straub, Thomas Whelan, Lingni Ma, Yufan Chen, Erik Wijmans, Simon Green, Jakob~J. Engel, Raul Mur-Artal, Carl Ren, Shobhit Verma, Anton Clarkson, Mingfei Yan, Brian Budge, Yajie Yan, Xiaqing Pan, June Yon, Yuyang Zou, Kimberly Leon, Nigel Carter, Jesus Briales, Tyler Gillingham, Elias Mueggler, Luis Pesqueira, Manolis Savva, Dhruv Batra, Hauke~M. Strasdat, Renzo~De Nardi, Michael Goesele, Steven Lovegrove, and Richard Newcombe.
\newblock The {R}eplica dataset: A digital replica of indoor spaces.
\newblock \emph{arXiv preprint arXiv:1906.05797}, 2019.

\bibitem[Sturm et~al.(2012)Sturm, Engelhard, Endres, Burgard, and Cremers]{sturm12tum}
J.~Sturm, N.~Engelhard, F.~Endres, W.~Burgard, and D.~Cremers.
\newblock A benchmark for the evaluation of rgb-d slam systems.
\newblock In \emph{Proc. of the International Conference on Intelligent Robot Systems (IROS)}, Oct. 2012.

\bibitem[Sucar et~al.(2021)Sucar, Liu, Ortiz, and Davison]{sucar2021imap}
Edgar Sucar, Shikun Liu, Joseph Ortiz, and Andrew~J Davison.
\newblock imap: Implicit mapping and positioning in real-time.
\newblock In \emph{Proceedings of the IEEE/CVF international conference on computer vision}, pages 6229--6238, 2021.

\bibitem[Tian et~al.(2021)Tian, Khosoussi, Rosen, and How]{tian2021distributed}
Yulun Tian, Kasra Khosoussi, David~M Rosen, and Jonathan~P How.
\newblock Distributed certifiably correct pose-graph optimization.
\newblock \emph{IEEE Transactions on Robotics}, 37\penalty0 (6):\penalty0 2137--2156, 2021.

\bibitem[Wolkowicz et~al.(2012)Wolkowicz, Saigal, and Vandenberghe]{wolkowicz2012handbook}
Henry Wolkowicz, Romesh Saigal, and Lieven Vandenberghe.
\newblock \emph{Handbook of semidefinite programming: theory, algorithms, and applications}, volume~27.
\newblock Springer Science \& Business Media, 2012.

\bibitem[Yang(2024)]{yang24book-sdp}
Heng Yang.
\newblock Semidefinite optimization and relaxation.
\newblock \emph{Working draft edition, \url{https://hankyang.seas.harvard.edu/Semidefinite/}}, 2024.

\bibitem[Yang and Carlone(2022)]{yang2022certifiably}
Heng Yang and Luca Carlone.
\newblock Certifiably optimal outlier-robust geometric perception: Semidefinite relaxations and scalable global optimization.
\newblock \emph{IEEE transactions on pattern analysis and machine intelligence}, 45\penalty0 (3):\penalty0 2816--2834, 2022.

\bibitem[Yang et~al.(2020)Yang, Shi, and Carlone]{yang2020teaser}
Heng Yang, Jingnan Shi, and Luca Carlone.
\newblock Teaser: Fast and certifiable point cloud registration.
\newblock \emph{IEEE Transactions on Robotics}, 37\penalty0 (2):\penalty0 314--333, 2020.

\bibitem[Yang et~al.(2024{\natexlab{a}})Yang, Kang, Huang, Xu, Feng, and Zhao]{yang2024depth}
Lihe Yang, Bingyi Kang, Zilong Huang, Xiaogang Xu, Jiashi Feng, and Hengshuang Zhao.
\newblock Depth anything: Unleashing the power of large-scale unlabeled data.
\newblock In \emph{Proceedings of the IEEE/CVF Conference on Computer Vision and Pattern Recognition}, pages 10371--10381, 2024{\natexlab{a}}.

\bibitem[Yang et~al.(2024{\natexlab{b}})Yang, Kang, Huang, Zhao, Xu, Feng, and Zhao]{yang24depthv2}
Lihe Yang, Bingyi Kang, Zilong Huang, Zhen Zhao, Xiaogang Xu, Jiashi Feng, and Hengshuang Zhao.
\newblock Depth anything v2.
\newblock \emph{arXiv:2406.09414}, 2024{\natexlab{b}}.

\bibitem[Yin and Hu(2024)]{yin24metric3d}
Wei Yin and Mu~Hu.
\newblock {Metric3D}: A toolbox for zero-shot metric depth estimation.
\newblock \url{https://github.com/YvanYin/Metric3D}, 2024.

\bibitem[Yu and Yang(2024)]{yu2024sim}
Xihang Yu and Heng Yang.
\newblock Sim-sync: From certifiably optimal synchronization over the 3d similarity group to scene reconstruction with learned depth.
\newblock \emph{IEEE Robotics and Automation Letters}, 2024.

\bibitem[Zhu et~al.(2022)Zhu, Peng, Larsson, Xu, Bao, Cui, Oswald, and Pollefeys]{zhu2022nice}
Zihan Zhu, Songyou Peng, Viktor Larsson, Weiwei Xu, Hujun Bao, Zhaopeng Cui, Martin~R Oswald, and Marc Pollefeys.
\newblock Nice-slam: Neural implicit scalable encoding for slam.
\newblock In \emph{Proceedings of the IEEE/CVF conference on computer vision and pattern recognition}, pages 12786--12796, 2022.

\end{thebibliography}

\clearpage
\appendices

\section{Proofs}
\label{sec:app_proofs}

\begin{table*}[ht]
    \centering
    \caption{\textnormal{Results of different Depth Estimation Model on the Mip-Nerf datasets}. }
    \label{tab:depthcomparison}
    \begin{adjustbox}{width=\linewidth}
        \begin{tabular}{l|c|c|c|c|c|c|c|c|c|c|c|c|}
            \toprule
            Method & \multicolumn{3}{c|}{Unidepth} & \multicolumn{3}{c|}{Depth-pro}& \multicolumn{3}{c|}{DepthAnything-v2} & \multicolumn{3}{c|}{Metric3D-v2} \\
            \midrule
            Metrics & \stackon{Solver\,Time}{Processing\,Time} & \stackon{ATE-T}{ATE-R} & \stackon{RPE-T}{RPE-R} & \stackon{Solver\,Time}{Processing\,Time} & \stackon{ATE-T}{ATE-R} & \stackon{RPE-T}{RPE-R} & \stackon{Solver\,Time}{Processing\,Time} & \stackon{ATE-T}{ATE-R} & \stackon{RPE-T}{RPE-R} & \stackon{Solver\,Time}{Processing\,Time} & \stackon{ATE-T}{ATE-R} & \stackon{RPE-T}{RPE-R} \\
            \midrule\midrule
            \stackon{-187865}{Kitchen-279}
            & \stackon{ $\bf 0.8 + 19.19 $}{$ 229.82 $} & \stackon{$ 0.018 $}{$ 0.154 ^{\circ}$} & \stackon{$ 0.027 $}{$ 0.201 ^{\circ}$}
            & \stackon{$ 1.13 + 21.54 $}{$ 436.19 $} & \stackon{$ \bf0.006 $}{$ \bf0.051 ^{\circ}$} & \stackon{$\bf 0.008 $}{$ \bf0.062 ^{\circ}$}
            & \stackon{$ 0.9 + 29.52 $}{$ 224.25 $} & \stackon{$ 0.053 $}{$ 0.384 ^{\circ}$} & \stackon{$ 0.072 $}{$ 0.478 ^{\circ}$}
            & \stackon{$ 1.66 + 107.85 $}{$\bf 216.0 $} & \stackon{$ 0.064 $}{$ 0.621 ^{\circ}$} & \stackon{$ 0.099 $}{$ 0.785 ^{\circ}$}

            \\
            \midrule
            \stackon{-106858}{Garden-185}
            & \stackon{$ 0.65 + 3.98 $}{$ \bf174.57 $} & \stackon{$\bf 0.002 $}{$ 0.021 ^{\circ}$} & \stackon{$ \bf0.003 $}{$ \bf0.025 ^{\circ}$}
            & \stackon{$ 0.54 + 16.96 $}{$ 322.73 $} & \stackon{$ 0.007 $}{$ 0.061 ^{\circ}$} & \stackon{$ 0.01 $}{$ 0.035 ^{\circ}$}
            & \stackon{$ \bf 0.61 + 3.82 $}{$ 179.88 $} & \stackon{$ \bf0.002 $}{$ \bf0.02 ^{\circ}$} & \stackon{$ \bf0.003 $}{$ \bf0.025 ^{\circ}$}
            & \stackon{$ 1.01 + 21.74 $}{$ 197.72 $} & \stackon{$ 0.007 $}{$ 0.059 ^{\circ}$} & \stackon{$ 0.009 $}{$ 0.035 ^{\circ}$}

            \\
            \midrule
            \stackon{-41866}{Bicycle-194}
            & \stackon{$ 0.58 + 22.47 $}{$ \bf145.75 $} & \stackon{$ 0.019 $}{$ 0.229 ^{\circ}$} & \stackon{$ 0.028 $}{$ 0.219 ^{\circ}$}
            & \stackon{$ 0.57 + 11.61 $}{$ 299.62 $} & \stackon{$ \bf0.003 $}{$ 0.05 ^{\circ}$} & \stackon{$ 0.006 $}{$ 0.039 ^{\circ}$}
            & \stackon{$\bf 0.53 + 3.44 $}{$ 152.37 $} & \stackon{$ \bf0.003 $}{$ \bf0.034 ^{\circ}$} & \stackon{$ \bf0.004 $}{$ \bf0.035 ^{\circ}$}
            & \stackon{$ 1.11 + 7.0 $}{$ 168.63 $} & \stackon{$ 1.123 $}{$ 19.257 ^{\circ}$} & \stackon{$ 1.466 $}{$ 7.533 ^{\circ}$}

            \\
            \midrule
            \stackon{-111783}{Room-311}
            & \stackon{$ \bf1.01 + 24.35 $}{$ 223.44 $} & \stackon{$ 0.002 $}{$ 0.06 ^{\circ}$} & \stackon{$ \bf0.002 $}{$ 0.084 ^{\circ}$}
            & \stackon{$ 1.18 + 33.49 $}{$ 429.43 $} & \stackon{$ \bf0.001 $}{$ \bf0.053 ^{\circ}$} & \stackon{$ \bf0.002 $}{$ \bf0.06 ^{\circ}$}
            & \stackon{$ 0.99 + 37.35 $}{$ 190.9 $} & \stackon{$ 0.002 $}{$ 0.067 ^{\circ}$} & \stackon{$ \bf0.002 $}{$ 0.095 ^{\circ}$}
            & \stackon{$ 1.08 + 86.68 $}{$ \bf180.55 $} & \stackon{$ 0.052 $}{$ 2.54 ^{\circ}$} & \stackon{$ 0.083 $}{$ 1.584 ^{\circ}$}

            \\

            \bottomrule
        \end{tabular}
    \end{adjustbox}
\end{table*}

\subsection{Proof of Proposition \ref{prop:formulation_scale_rotation_only}}\label{app:proof_scale_rotation_only}

\begin{proof}
    We begin by expressing the objective function in \prettyref{eq:sba} in its vectorized form: 
    \bea 
    &\displaystyle \sum_{(i,k) \in \calE} w_{ik} \norm{ R_i (s_i\tilde{u}_{ik}) + t_i - p_k } ^2 \nonumber \\
    = & \displaystyle \sum_{(i,k) \in \calE} w_{ik} \norm{ (\tilde{u}_{ik} \tran \kron \eye_3) \vectorize{s_i R_i} + t_i -p_k} ^2 \nonumber 
    \eea
    where $w_{ik}$ denotes the weight assigned to each term.  

    Let $e_i \in \Real{N}$ be a vector of zeros except for the $i$-th entry, which is set to 1. Similarly, let $\bar{e}_k \in \Real{M}$ be a vector of zeros except for the $k$-th entry, which is set to 1. Furthermore, define the following concatenated vectors:  
    $t = [t_1;\dots;t_N] \in \Real{3N}$ for translations,
    $p = [p_1;\dots;p_M] \in \Real{3M}$ for landmark positions,  
    $\displaystyle r = [\vectorize{s_1 R_1};\dots;\vectorize{s_N R_N}] \in \Real{9N}$ for vectorized scaled rotations.  

    With these definitions, we simplify the objective function as $L(t,p,r)$.
    \bea
     \sum_{(i,k) \in \calE} w_{ik}|| ((e_i \tran \kron \tilde{u}_{i,k}) \kron \eye_3) r + (e_i\tran \kron \eye_3)t - (\bar{e}_k \tran \kron \eye_3)p ||^2 \nonumber 
    \eea
    \bea
    = & \displaystyle r\tran \parentheses{ Q_1 \kron \eye_3} r + t\tran \parentheses{ Q_2 \kron \eye_3} t + p\tran \parentheses{ Q_3 \kron \eye_3} p + \nonumber\\
    & 2r \tran \parentheses{ V_1 \kron \eye_3} t - 2r \tran \parentheses{ V_2 \kron \eye_3} p - 2t \tran \parentheses{ V_3 \kron \eye_3} p \nonumber
    \eea
    where $Q_1,Q_2,Q_3,V_1,V_2,V_3$ are as follows
    \bea 
    Q_1 = & \displaystyle \sum_{(i,k) \in \calE} w_{ik} (e_i  \kron \tilde{u}_{i,k} ) (e_i \tran \kron p_{i,k} \tran ) \in \Real{3N \times 3N},  \\
    Q_2 = & \displaystyle \sum_{(i,k) \in \calE} w_{ik} (e_i e_i \tran) \in \Real{N \times N}, \\
    Q_3 = & \displaystyle \sum_{(i,k) \in \calE} w_{ik} (\bar{e}_k \bar{e}_k \tran) \in \Real{M \times M}, \\
    V_1 = & \displaystyle \sum_{(i,k) \in \calE} w_{ik} (e_i \kron \tilde{u}_{i,k} ) e_i \tran \in \Real{3N \times N}, \\
    V_2 = & \displaystyle \sum_{(i,k) \in \calE} w_{ik} (e_i \kron \tilde{u}_{i,k} ) \bar{e}_k \tran \in \Real{3N \times M}, \\
    V_3 = & \displaystyle \sum_{(i,k) \in \calE} w_{ik} e_i \bar{e}_k \tran \in \Real{N \times M}, 
    \eea
    We now solve out $t$ and $p$ by taking the gradient of $L(t,p,r)$ \wrt $t$ and $p$ respectively.  Let $\displaystyle y = \bmat{c} t \\ p \emat $. We can represent $L(t,p,r)$ using $y,r$ as follows:
    
    \begin{align}\label{eq:problem_yr_only}
    L(y,r) =\displaystyle r\tran \parentheses{ Q_1 \kron \eye_3} r - 2r\tran \parentheses{ V_{tp} \kron \eye_3} y \nonumber \\+ y\tran \parentheses{ Q_{tp} \kron \eye_3} y
    \end{align}
    where $Q_{tp}, V_{tp}$ are as follows:
    \bea
    Q_{tp} = &\bmat{cc} Q_2 & -V_3 \\ -V_3 \tran & Q_3 \emat \in \Real{(N+M)\times(N+M)} \label{eq:app-Qtp-restate}\\ 
    V_{tp} = & \bmat{cc} -V_1 & V_2 \emat \in \Real{3N \times (N+M)}.
    \eea
    set $\nabla_y L(y,r) = 0$ we obtain
    \bea\label{eq:nabla_y_zero}
    (Q_{tp} \kron \eye_3) {y} = (V_{tp} \kron \eye_3)\tran r.
    \eea

    Let $\calG$ denote the view-graph. Next we prove $Q_{tp}$ can be represented as:
    \bea
    Q_{tp} = L(\calG) 
    \eea     
    where $L(\calG)$ is the Laplacian of $\calG$:
    \begin{lemma}
        \label{lemma:laplacian}
        $Q_{tp}$ is the Laplacian of $\calG$. \\
        \textit{Proof:} Note that $\calG$ is a weighted undirected graph. Calling $\delta(q)$ for the set of edges incident to a vertex $q$, and $w_e = w_{qp}$ for $e = (q,p)$, the Laplacian of $\calG$ is:
            \bea
            \label{eq:laplacian}
            L(\calG)_{qp} = 
            \begin{cases}
                \sum_{e \in \delta(q)} w_e & \text{if } q = p, \\
                - w_{qp} & \text{if } (q,p)\in \calE, \\
                0 & \text{if } (i,j) \notin \calE.
            \end{cases}
            \eea
        On the other hand, by expanding \prettyref{eq:app-Qtp-restate} and compare it with \prettyref{eq:laplacian}, we finish the proof. \qed
    \end{lemma}
    
     Calling $\bar{y} = [t_2;\dots;t_{N};p_1;\dots;p_{N+M}] \in \Real{3(N+M)-3}$ and $\bar{Q}_{tp} = [c_2;\dots;c_{N+M}] \in \Real{(N+M) \times (N+M-1)}$ where $c_i$ is the i-th column of $Q_{tp}$. We prove when fix $s_1 = 1$, $R_1 = \eye_3$, $t_1 = 0$, $y$ has a unique solution:
    \bea
    (\bar{Q}_{tp} \kron \eye_3) {\bar{y}} = -(V_{tp} \kron \eye_3)\tran r
    \eea
    Since $\rank{L(\calG)}=N+M-1$ and $\sum_{i=1,..,N} c_i = \mathbf{0}_{N}$, then $span(c_1,...,c_n) = span(c_2,...,c_n)$ and $\rank{\bar{Q}_{tp}}=N+M-1$ which implies that $\bar{Q}_{tp}$ has full column rank. Hence, by taking inverse and rearrange, we obtain
    \bea 
    \bar{y} = & \parentheses{\bar{A} \kron \eye_3} r
    \eea 
    where
    \bea
    \bar{A} = & -(\bar{Q}_{tp} \tran \bar{Q}_{tp})\inv \bar{Q}_{tp} \tran V_{tp} \tran
    \eea
    Together with $t_1 = 0$,
    \bea
    y = & \parentheses{A \kron \eye_3} r
    \eea
    with
    \bea
    A = & \bmat{c} \zero_{1 \times 3N} \\ -(\bar{Q}_{tp} \tran \bar{Q}_{tp})\inv \bar{Q}_{tp} \tran V_{tp} \tran \emat \in \Real{(N+M) \times 3N}
    \eea
    Now we have a closed form solution of $y$. Plug in the solution of $y$ into \eqref{eq:problem_yr_only}, we obtain:
    \bea
    L(r) = & \displaystyle r \tran \parentheses{\parentheses{A \tran Q_{tp} A + V_{tp}A + A\tran V_{tp}\tran + Q_1} \kron \eye_3} r \nonumber
    \eea
    Note that
    \bea
    r = [\vectorize{s_1 R_1};\dots;\vectorize{s_N R_N}] = \vectorize{U} \nonumber
    \eea
    Then $L(r)$ is equivalent to
    \bea
    \label{loss_R}
    \vectorize{U} \tran \parentheses{\parentheses{A \tran Q_{tp} A + V_{tp}A + A\tran V_{tp}\tran + Q_1} \kron \eye_3} \vectorize{U}\nonumber
    \eea
    Rewriting this in a more compact matrix representation yields:
    \begin{align}
        \rho^\star = \min_{R} &~\trace{QU \tran U} \\
        Q &\coloneqq  A \tran Q_{tp} A + V_{tp}A + A\tran V_{tp}\tran + Q_1 \in \sym{3N}
    \end{align}
    concluding the proof.
\end{proof}

\subsection{Proof of \prettyref{prop:qcqp}}\label{app:proof_qcqp}

\begin{proof}
    We show that under \prettyref{eq:determinant}, if $U^\star$ is a global optimizer of \prettyref{eq:problem-qcqp}, then it is also a global optimizer of \prettyref{eq:problem_scale_rotation_only}.  

    First, note that \prettyref{eq:problem-qcqp} is a relaxation of \prettyref{eq:problem_scale_rotation_only}, because $\Othree \subseteq \SOthree$. Thus, we have: 
    \bea\label{eq:qcqp_relaxation} 
      \trace{Q(U^\star)\tran U^\star} = \rho^\star_{\qcqp} \leq \rho^\star.
    \eea  
    From \prettyref{eq:determinant} and the definition $\bar{R}^\star_i = s^\star_i R^\star_i$, we obtain  
    \bea
      \det(R^\star_i) = \frac{\det(\bar{R}^\star_i)}{(s^\star_i)^3} > 0.
    \eea  
    This implies that $U^\star$ is also a feasible solution to \prettyref{eq:problem_scale_rotation_only}. By \prettyref{eq:qcqp_relaxation}, $U^\star$ is therefore the global optimizer of \prettyref{eq:problem_scale_rotation_only}.
\end{proof}

\subsection{Proof of \prettyref{prop:sdprelaxation}}\label{app:proof_sdp_relaxation}

\begin{proof}
    We first establish that \prettyref{eq:problem-sdp} is a relaxation of \prettyref{eq:problem-qcqp}. Let $X = U\tran U$. By definition, we have $x\tran X x = (Ux)\tran Ux \geq 0$ for all $x \in \Real{3N}$, which implies that $X \succeq 0$.  

    Furthermore, since $\bar{R}_i \in \sOthree$, it follows that $\bar{R}^\tran_i \bar{R}_i = s_i^2 \eye_3 \triangleq \alpha \eye_3$. This ensures that the feasible set of \prettyref{eq:problem-sdp} is broader than that of \prettyref{eq:problem-qcqp}, confirming that it is indeed a relaxation.

    Next we prove if $\rank{X^\star} = 3$, we can extract the global minimizer of \prettyref{eq:problem-qcqp} from $X^\star$. Let $X^\star = (\bar{U}^\star)\tran \bar{U}^\star$, then $\bar{U}^\star$ is rank 3. We can decompose $\bar{U}^\star$ as 
    \bea
    \bar{U}^\star = \bmat{cccc} \bar{R}^\star_1 & \bar{R}^\star_2 & \dots & \bar{R}^\star_N \emat
    \eea
    Since $X$ is feasible, we have $\bar{R}^\star_i \bar{R}^\star_i = \alpha \eye_3 = s_i^2 \eye_3$, which implies that $\bar{R}^\star_i / s_i \in \Othree$.  

    Define $U^\star = \bar{R}_1^{\star T} \bar{U}^\star$, making $U^\star$ a feasible solution to \prettyref{eq:problem-qcqp}. Moreover, since  
    \bea
    \trace{Q(U^\star)\tran U^\star} = \rho^\star_{\sdp} \leq \rho^\star_{\qcqp},  
    \eea  
    it follows that $U^\star$ is the global minimizer of \prettyref{eq:problem-qcqp}.
\end{proof}

\subsection{Proof of LICQ condition}\label{app:proof_licq}

\begin{proof}    
    Since $U_r = [R_1,\dots,R_n], R_k \in \mathbb{R}^{r \times 3}$ satisfies \eqref{eq:problem-bm}, we have $R_k^T R_k = \alpha_k I_3, \alpha_k > 0$. Therefore, $\text{rank}(R_k) = 3,\forall k$, i.e., $R_k$ is column full-rank. 

    We show the structure of the $\{A_i \}_{i=1}^m$ matrices guarantees LICQ. Recall $\{A_i \}_{i=1}^m$ enforces the diagonal $3\times 3$ blocks of $X$ in (19) to be scaled identity matrices (for the first block an additional constraint enforces the diagonal entries to be $1$). Therefore, the $k$-th group of $\{ A_i\}_{i=1}^m$ has the following form
\begin{equation}
\adjustbox{max width=0.6\linewidth}{$
    A_{5k - 4 + \ell} = 
\begin{bmatrix}
0 & \cdots & 0 & \cdots & 0 \\
\vdots & \ddots & \vdots & \ddots & \vdots \\
0 & \cdots & B^\ell & \cdots & 0 \\
\vdots & \ddots & \vdots & \ddots & \vdots \\
0 & \cdots & 0 & \cdots & 0 \\
\end{bmatrix}, \ell = 1,...,5
$} \nonumber
\end{equation}
where the matrix is all zero except in the $k$-th $3\times 3$ block:
\begin{equation}
\adjustbox{max width=0.95\linewidth}{$
B^1 = \begin{bmatrix}
    1 & 0 & 0 \\
    0 & -1 & 0 \\
    0 & 0 & 0 \\
\end{bmatrix},B^2 = \begin{bmatrix}
    0 & 0 & 0 \\
    0 & 1 & 0 \\
    0 & 0 & -1 \\
\end{bmatrix},
B^3 = \begin{bmatrix}
    0 & 1 & 0 \\
    1 & 0 & 0 \\
    0 & 0 & 0 \\
\end{bmatrix},B^4 = \begin{bmatrix}
    0 & 0 & 1 \\
    0 & 0 & 0 \\
    1 & 0 & 0 \\
\end{bmatrix},B^5 = \begin{bmatrix}
    0 & 0 & 0 \\
    0 & 0 & 1 \\
    0 & 1 & 0 \\
\end{bmatrix}.
$}\nonumber
\end{equation}
Due to this structure, 
$A_{5k - 4 + \ell} U_r^T$ is entirely zero except for rows \(3k-2\) through \(3k\). So $A_{5k - 4 + \ell} U_r^T$ for different values of \(k\) are clearly linearly independent since their nonzero regions do not overlap. Hence, we only need to investigate LICQ given a fixed $k$. 

Proof by contradiction: suppose there exists $\gamma_\ell,\ell=1,\dots,5$ such that 
$$
(\sum_{\ell=1}^5 \gamma_\ell B^\ell) R_k^T = 0,
$$
then, because $R_k^T$ is row full-rank, we have $\sum_{\ell=1}^5 \gamma_\ell B^\ell=0$, and thus $\gamma_\ell =0, \forall \ell$. This proves that $\{A_i U_r^T\}_{i=1}^m$ must be linearly independent.
\end{proof}

\section{Depth Models}
\label{app:depthcomparison}

We compare the performance of different depth estimation models on the Mip-Nerf datasets. The models we consider are \textbf{Unidepth}~\cite{piccinelli2024unidepth}, \textbf{Depth-pro}~\cite{bochkovskii2024depth}, \textbf{DepthAnything-V2}~\cite{yang24depthv2}, and \textbf{Metric3D-v2}~\cite{yin24metric3d}. We evaluate these models using ATE/RTE and runtime as detailed in \prettyref{sec:exp}. The results are shown in \prettyref{tab:depthcomparison}.

All four models produce accurate results, demonstrating the robustness of our \xmsfm pipeline to different depth models. Unidepth is the fastest, while DepthAnything-V2 and Metric3D-v2 are slightly slower but occasionally more accurate. Depth-Pro is the slowest yet consistently stable in accuracy.


\begin{figure*}[!t]
    \begin{center}
		\begin{tabular}{cc}
    \begin{minipage}{0.34\textwidth} 
        \begin{adjustbox}{width=\linewidth}
            \begin{tikzpicture}
                \pgfplotstableread{
                    Label  Match Index  Depth Filter Matrix Solver
                    R1     0 0 0 0 0 0
                    R2     0 0 0 0 0 0
                    R3     0 0 0 0 0 0
                    R4     0 0 0 0 0 0
                    R5     0 0 0 0 0 0
                    A1     349.77 607.54 55.16 0 130.26 8.5
                    A2     379.45 531.55 63.0 0 182.6 7.37
                    A3     168.36 445.49 40.72 0 75.59 4.6
                    A4     140.53 525.81 40.98 0 52.12 23.1
                    A5     0       0      0     0      0     0
                    B1     349.77 607.54 53.63 727.3 130.0 16.26
                    B2     379.45 531.55 63.7 511.36 165.06 6.18
                    B3     168.36 445.49 43.41 480.9 71.59 5.61
                    B4     140.53 525.81 37.18 543.33 52.36 20.72
                    B5     0       0      0         0     0         0
                }\testdata

                \begin{axis}[
                    ybar stacked,
                    ymin=0,
                    ymax=2000,
                    xtick={1,6,11},  
                    xticklabels={\,, \xmdouble, \textsc{Filter + }\xmdouble}, 
                    legend style={cells={anchor=west}, legend pos=north west, font = \scriptsize},
                    bar width=10pt, 
                    enlarge x limits=0.1, 
                    reverse legend=true,
                ]
    
                \addplot [fill=green!80] 
                    table [x expr=\coordindex, y=Match] 
                    {\testdata};
                \addlegendentry{Match}
            
                \addplot [fill=blue!60] 
                    table [x expr=\coordindex, y=Index]
                    {\testdata};
                \addlegendentry{Index}
            
                \addplot [fill=purple!60] 
                    table [x expr=\coordindex, y=Depth]
                    {\testdata};
                \addlegendentry{Depth}
    
                \addplot [fill=orange!60] 
                    table [x expr=\coordindex, y=Filter] {\testdata};
                \addlegendentry{Filter}
    
                \addplot [fill=yellow!60] 
                    table [x expr=\coordindex, y=Matrix] {\testdata};
                \addlegendentry{Matrix}
    
                \addplot [fill=red!60] 
                table [x expr=\coordindex, y=Solver] {\testdata};
                \addlegendentry{Solver}
            
                \end{axis}
                
            \end{tikzpicture}
        \end{adjustbox}
        \caption{Breakdown Results on the Replica dataset (2000 frames)}
        \label{fig:replica-time-2000}
    \end{minipage}
    \hfill 

    \begin{minipage}{0.32\textwidth}
        \begin{adjustbox}{width=\linewidth}
            \begin{tikzpicture}
                \pgfplotstableread{
        Label  Match   Index  Depth Filter Matrix Solver
        R1     0 0 0 0 0 0
        R2     0 0 0 0 0 0
        R3     0 0 0 0 0 0
        R4     0 0 0 0 0 0
        R5     0 0 0 0 0 0
        A1     1.71 4.69 1.57 0 2.54 0.12
        A2     1.99 6.76 1.62 0 2.71 0.14
        A3     3.58 5.09 1.53 0 2.55 0.15
        A4     1.36 3.01 1.52 0 2.32 0.13
        A5     0       0      0     0      0     0 
        C1     1.71 4.69 1.6 3.34 2.58 0.12
        C2     1.99 6.76 1.6 3.21 2.62 0.14
        C3     3.58 5.09 1.56 3.49 2.81 0.12
        C4     1.36 3.01 1.52 2.8 2.41 0.15
        C5     0       0      0         0     0  0        
        }\testdata

                \begin{axis}[
                    ybar stacked,
                    ymin=0,
                    ymax=20,
                    xtick={1, 6, 11},  
                    xticklabels={\,,\xmdouble, \textsc{Filter + }\xmdouble}, 
                    legend style={cells={anchor=west}, legend pos=north west, font = \scriptsize},
                    bar width=10pt, 
                    enlarge x limits=0.1, 
                    reverse legend=true,
                ]
        
                 \addplot [fill=green!80] 
                 table [x expr=\coordindex, y=Match] 
                 {\testdata};
             \addlegendentry{Match}
         
             \addplot [fill=blue!60] 
                 table [x expr=\coordindex, y=Index]
                  {\testdata};
             \addlegendentry{Index}
         
             \addplot [fill=purple!60] 
                 table [x expr=\coordindex, y=Depth]
                 {\testdata};
             \addlegendentry{Depth}
        
             \addplot [fill=orange!60] 
                 table [x expr=\coordindex, y=Filter] {\testdata};
             \addlegendentry{Filter}
        
             \addplot [fill=yellow!60] 
                 table [x expr=\coordindex, y=Matrix] {\testdata};
             \addlegendentry{Matrix}
        
             \addplot [fill=red!60] 
             table [x expr=\coordindex, y=Solver] {\testdata};
             \addlegendentry{Solver}
            
                \end{axis}
            \end{tikzpicture}
        \end{adjustbox}
        \caption{Breakdown Results on the Replica dataset (100 frames)}
        \label{fig:replica-time-100}
    \end{minipage}
    \hfill 

    \begin{minipage}{0.32\textwidth}
        \begin{adjustbox}{width=\linewidth}
            \begin{tikzpicture}
                \pgfplotstableread{
                    Label  Match Index  Depth Filter Matrix XM Ceres
                    R1     0 0 0 0 0 0 0
                    B1     87.65 36.58 55.22 16.06 30.66 0.8 19.19
                    B2     53.77 37.93 59.62 8.8 14.46 0.65 3.98
                    B3     56.23 23.99 50.69 2.17 12.67 0.58 22.47
                    B4     61.77 34.49 59.5 18.72 24.6 1.01 24.35
                }\testdata

                \begin{axis}[
                    ybar stacked,
                    ymin=0,
                    ymax=300,
                    xtick={1,2,3,4},  
                    xticklabels={Kitchen,Garden,Bicycle,Room}, 
                    legend style={cells={anchor=west}, legend pos=north west, font = \scriptsize},
                    bar width=10pt, 
                    enlarge x limits=0.1, 
                    reverse legend=true,
                ]
    
                \addplot [fill=green!80] 
                    table [x expr=\coordindex, y=Match] 
                    {\testdata};
                \addlegendentry{Match}
            
                \addplot [fill=blue!60] 
                    table [x expr=\coordindex, y=Index]
                    {\testdata};
                \addlegendentry{Index}
            
                \addplot [fill=purple!60] 
                    table [x expr=\coordindex, y=Depth]
                    {\testdata};
                \addlegendentry{Depth}
    
                \addplot [fill=orange!60] 
                    table [x expr=\coordindex, y=Filter] {\testdata};
                \addlegendentry{Filter}
    
                \addplot [fill=yellow!60] 
                    table [x expr=\coordindex, y=Matrix] {\testdata};
                \addlegendentry{Matrix}
    
                \addplot [fill=red!60] 
                table [x expr=\coordindex, y=XM] {\testdata};
                \addlegendentry{XM}
    
                \addplot [fill=pink!60] 
                table [x expr=\coordindex, y=Ceres] {\testdata};
                \addlegendentry{Ceres}

                \end{axis}
            \end{tikzpicture}
        \end{adjustbox}
        \caption{Breakdown Results on the Mip-Nerf dataset}
        \label{fig:mipnerf-time}
    \end{minipage}
    \vspace{-6mm}
\end{tabular}
\end{center}
\end{figure*}
\section{Breakdown plot of run time}
\label{app:breakdown-time}

We present time breakdown plots (\prettyref{fig:replica-time-2000}, \prettyref{fig:replica-time-100}, \prettyref{fig:mipnerf-time}), for three different cases. Depth estimation scales linearly with the number of frames, making it more significant in smaller datasets. Matching and indexing become dominant in larger datasets, as exhaustive matching time grows quadratically with the number of frames. Solver time remains consistently minimal across all datasets.  


\section{Scale Regularization}
\label{app:scale-regularization}

Due to inaccuracies in real-world data, the scale of frames $2,\dots,N$ is often significantly smaller than that of the first frame\footnote{Anchored to be 1.}. As a result, the global minimum tends to collapse all cameras $2,\dots,N$ into a single point to avoid large errors in their estimates. To mitigate this, we introduce a scale regularization term in the objective function of \prettyref{eq:problem-sdp}.  

\begin{align}\label{eq:problem-sdp-scale}
    \min_{X \in \sym{3N}} & \trace{QX} + \lambda \sum_{i=2}^{N} (X_{3i,3i}-1)^2 \nonumber\\
    \subject & \langle A_i,X \rangle = b_i, \forall i \in 1,\dots,5N+1\nonumber\\
    & X \succeq 0
\end{align}
This modification does not violate the previously established theorems; Slater’s Condition and strong duality remain valid, and the optimality condition remains unchanged. However, the $Z(y)$ matrix is modified to  

\bea
    Z(y) = Q - \sum_{i=1}^{5N+1} y_i A_i + \lambda \nabla(\sum_{i=2}^{N} (X_{3i,3i}-1)^2)
\eea
To prove this, we first write \prettyref{eq:problem-sdp-scale} in the standard form:
\begin{align}
\min ~~&\langle Q, X \rangle + F(X)\\
s.t. ~~&X \succ 0\\
& \langle A_i, X \rangle = b_i, ~~\forall i
\end{align}
$F(X)$ is a quadratic term defined by $\displaystyle F(X) = \lambda \sum_{i=2}^{N} (X_{3i,3i}-1)^2$. This is still a convex problem and the Lagrangian is:
\begin{align}
L(X, y, Z) =& \langle Q, X \rangle + F(X) - \\&\sum_i y_i (\langle A_i, X \rangle - b_i) - \langle Z, X \rangle \\
=& \langle Q - \sum_i y_i A_i - Z, X \rangle + F(X) + \sum_i y_i b_i
\end{align}
The KKT condition is:
\begin{align}
Z =& \ Q + \nabla F(X) - \sum_i y_i A_i \succ 0 \\
\langle A_i, X \rangle =& \ b_i, ~~\forall i \\
X \succ & \ 0 \\
\langle Z, X \rangle = & \ 0
\end{align}
Note that $L$ is linear in the off-diagonal elements of $X$ and quadratic in the diagonal elements. To ensure a finite minimum, we require $(\nabla L)_{ij} = 0$ for $i \neq j$. For the diagonal elements, achieving the minimum requires $(\nabla L)_{ii} = 0$. Thus, together, we obtain $\nabla L = 0$, i.e.
\begin{align}
Q - \sum_i y_i A_i - Z = -\nabla F(X)
\end{align}
From the definition of $F(x)$ we know $X_{3j,3j} = -\frac{(Q - \sum_i y_i A_i - Z)_{3j,3j}}{2\lambda} + 1$. So we plug in $X$ then we get the dual problem:
\begin{align}
\max ~~&\sum_i b_i y_i - \frac{\sum_{j=2}^N (C - \sum_i y_i A_i - Z)_{3j,3j}^2}{4\lambda} 
\\ & + \sum_{j=2}^N (Q - \sum_i y_i A_i - Z)_{3j,3j} \\
s.t. ~~&Z \succ 0\\
        ~~&(Q - \sum_i y_i A_i - Z)_{ij} = 0, ~~i\neq j .   
\end{align}


\section{Suboptimality for SDP}
\label{app:suboptimality}

In \prettyref{eq:suboptimality}, we established suboptimality when \prettyref{eq:problem-sdp} is solved to global optimality. However, in practice, numerical errors arise, and variations in gradient tolerance settings can lead to increased suboptimality. Here, we provide a rigorous proof for computing suboptimality in SDP problems and apply it directly to the scale regularization problem in \prettyref{eq:problem-sdp-scale}.  

\begin{theorem}
    Given the primal problem \prettyref{eq:problem-sdp-scale} we have 
    \bea
        \rho_\sdp \geq \rho^\star_\sdp \geq \text{max}(0,\lambda_{\text{min}}(Z)) \text{tr}(X) + \rho_{\text{dual}}
    \eea
\end{theorem}
\begin{proof}
    for all feasible $X$ we have:
    \begin{align}
    \langle Q, X \rangle + F(X) &\geq \langle Q - \sum_i y_i A_i + \nabla F(x) , X \rangle \\
    &+ \underbrace{\sum_i y_i b_i + F(X) - \langle \nabla F(X), X \rangle}_{\text{dual value}}\\
    &= \langle Z, X \rangle + \rho_{\text{dual}}\\
    &\geq \text{max}(0,\lambda_{\text{min}}(Z)) \text{tr}(X) + \rho_{\text{dual}}
    \end{align}
    and we have $\rho_\sdp \geq \rho^\star_\sdp$.
\end{proof}

$\rho_\sdp$ and $\text{max}(0, \lambda_{\text{min}}(Z)) \text{tr}(X) + \rho_{\text{dual}}$ can be directly evaluated within the algorithm. A small gap between them indicates that $\rho_\sdp$ is close to $\rho^\star_\sdp$, confirming that the problem has been solved to global optimality. We define this new suboptimality gap as:  

\bea \label{eq:suboptimality_new}
\eta = \frac{\hat{\rho} - \text{max}(0,\lambda_{\text{min}}(Z)) \text{tr}(X) - \rho_{\text{dual}}}{1 + |\hat{\rho}| + |\text{max}(0,\lambda_{\text{min}}(Z)) \text{tr}(X) + \rho_{\text{dual}}|}.
\eea

\section{Convergence and Noise Analysis}
\label{app:converge-noise}

\begin{figure}[!t]
    \centering
    \includegraphics[width=0.9\linewidth]{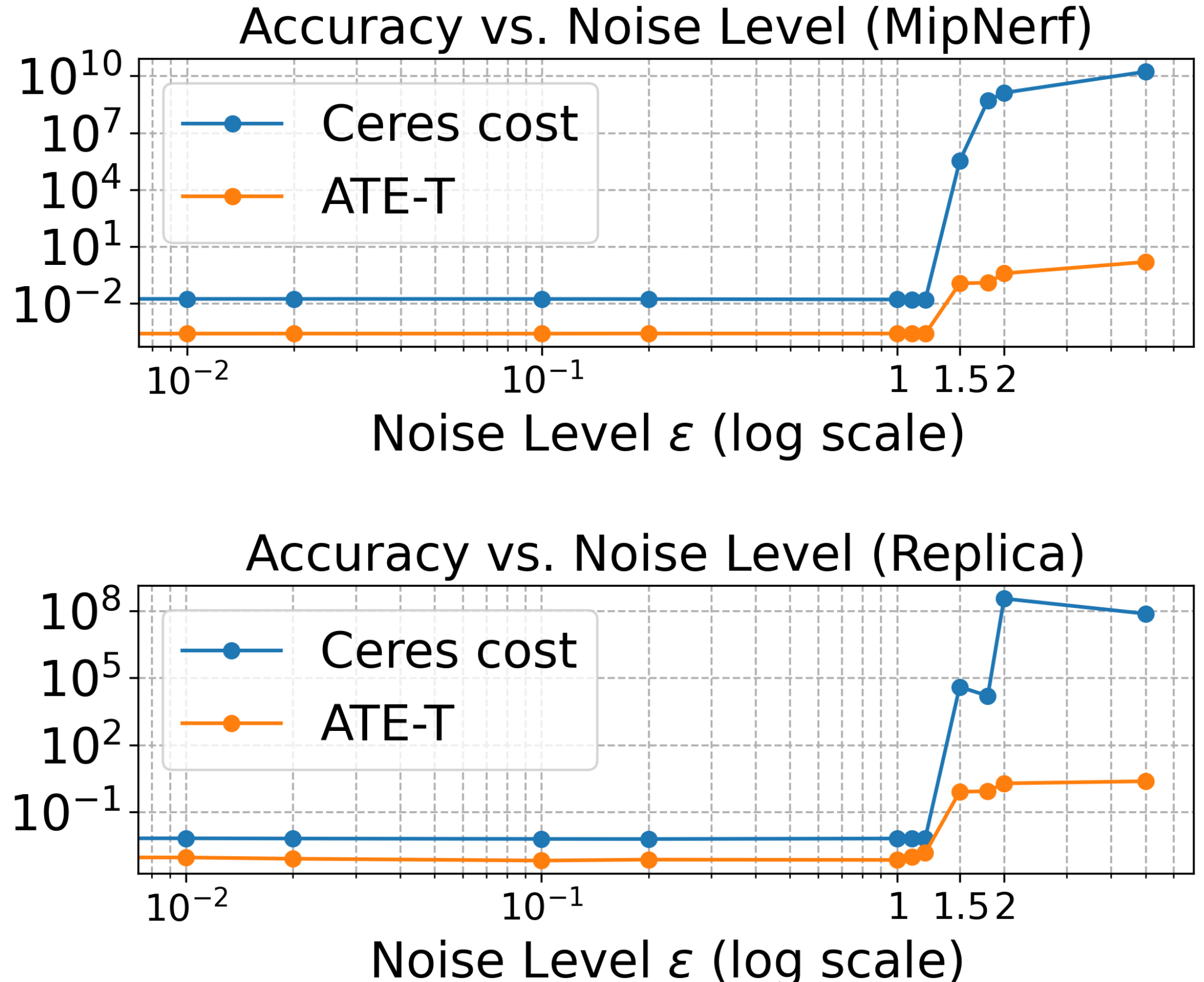}
    \caption{Accuracy (ATE-T) and the final Ceres objective value of Replica and Mip-NeRF datasets as the depth noise level $\epsilon$ increases. When $\epsilon<1.5$ they fall into the same local minima. The SDP relaxation breaks at $\epsilon=5$ for Replica and never for Mip-NeRF.}
    \label{fig:noise_test}
    \vspace{-3mm}
\end{figure}

\begin{figure}[!t]
    \centering
    \includegraphics[width=0.9\linewidth]{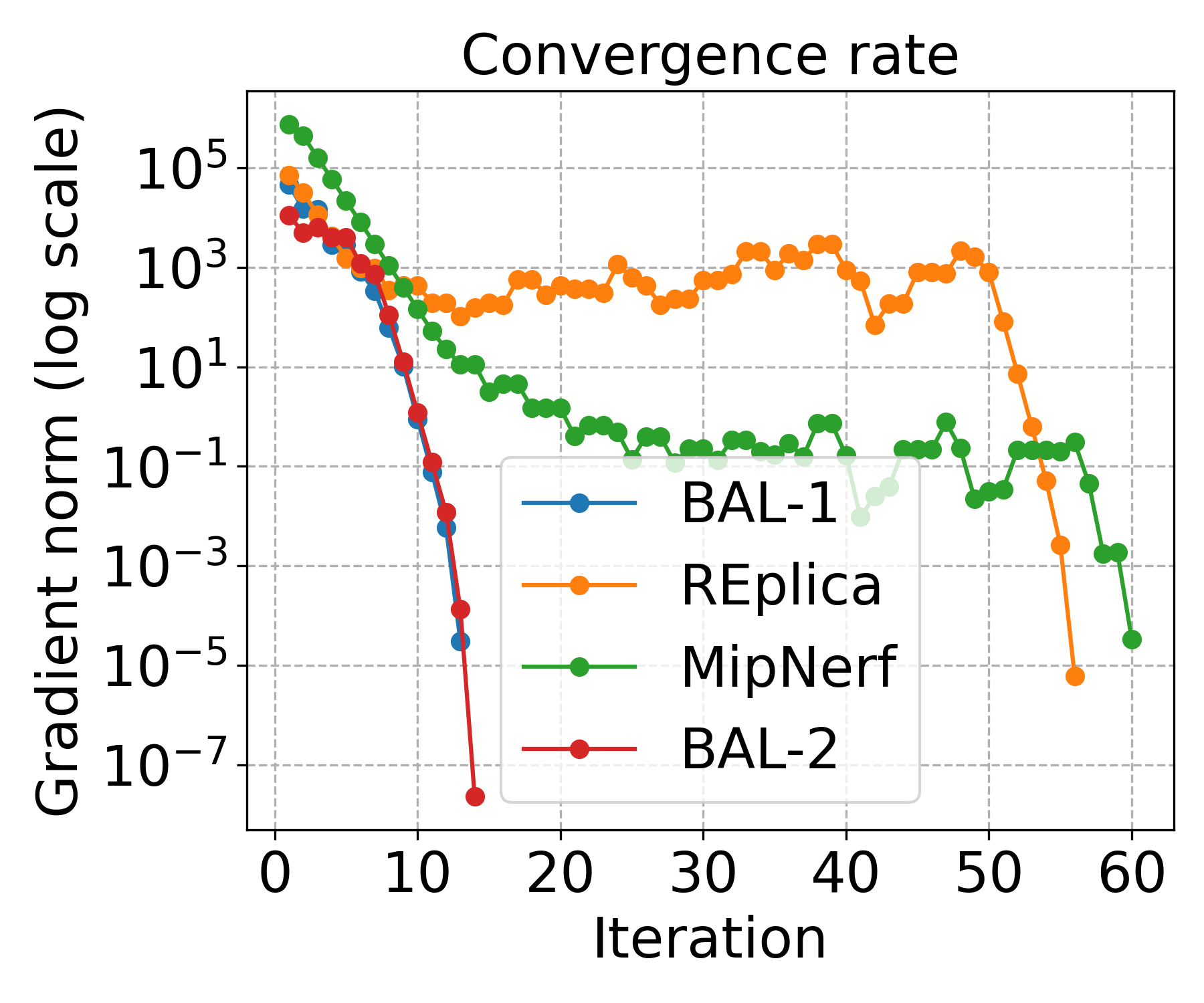}
    \caption{Convergence rate of XM on different datasets. BAL datasets show quadratic convergence. The other two enter quadratic convergence after a slow convergence region.}
    \label{fig:convergence_rate}
    \vspace{-3mm}
\end{figure}

\subsection{Convergence Rate}
We show the convergence rate of our algorithm on different datasets. From Riemannian trust-region
algorithm with truncated conjugate gradient (Rtr-tCG) we already know that it has local quadratic convergence \cite{boumal2023introduction,absil2008optimization}. Our experimental results (Fig.~\ref{fig:convergence_rate}) verified this.

\subsection{Noise Analysis} 
We test our pipeline to investigate tolerance for noise in depth estimation. Specifically, we add a noisy scaling factor $s$ to the depth of each observation, defined as
$s = (1+\epsilon)^x, x \sim \mathcal{U}(-1,1)$,
where \(\epsilon\) controls the noise level. Fig.~\ref{fig:noise_test} right plots the relationship between \(\epsilon\) and the reconstruction accuracy. The pipeline breaks at $\epsilon=1.5$ which shows robustness against inaccurate depth. The SDP relaxation does not break until $\epsilon=5$, a stronger level of robustness.

\section{Metrics}
\label{app:metrics}

The classical absolute Trajectory Error (ATE) and Relative
Pose Error (RPE) are defined as:

\begin{align}
    &\text{ATE}_{i} = \| \hat{T}_i - T_i \|_F^2, \\
    &\text{RPE}_{ij} = \| (\hat{T}_j \hat{T}_i)^{-1} T_i T_j^{-1} \|_F^2,
\end{align}
where $T_i$ represents ground truth pose and $\hat{T}_i$ represents estimated transformation. Because of the potential transformation between ground truth and estimation, an alignment between the two trajectories is required. The alignment is done by solving a point-cloud registration problem\footnote{We ignore the rotation in each camera frame and align position of cameras through TEASER \cite{yang2020teaser}.} between two clusters of cameras. 

\end{document}